\documentclass{article}
\usepackage{preprint,times}

\usepackage{graphicx}
\usepackage{tikz}
\usetikzlibrary{arrows.meta,positioning,calc,fit,backgrounds}
\usepackage{pgfplots}
\pgfplotsset{compat=1.18}
\usepackage{subcaption}

\providecolor{pqBoxGrey}{RGB}{226,226,226}
\providecolor{pqBoxBlue}{RGB}{186,216,240}
\providecolor{pqBoxGreen}{RGB}{195,230,195}
\providecolor{pqBoxPurple}{RGB}{219,206,240}
\providecolor{pqBoxPink}{RGB}{250,203,209}
\providecolor{pqDarkRed}{RGB}{160,20,28}

\providecolor{pqLink}{RGB}{26,62,196}

\providecolor{pqBlue}{RGB}{31,119,180}
\providecolor{pqOrange}{RGB}{255,127,14}
\providecolor{pqGreen}{RGB}{44,160,44}
\providecolor{pqRed}{RGB}{214,39,40}

\providecolor{pqPlaneI}{RGB}{186,216,240}
\providecolor{pqPlaneII}{RGB}{112,171,214}
\providecolor{pqPlaneIII}{RGB}{34,106,168}
\providecolor{pqPlaneEdge}{RGB}{26,86,140}
\providecolor{pqSteel}{RGB}{12,82,135}

\providecolor{pqOrangeFill}{RGB}{252,186,118}
\providecolor{pqOrangeEdge}{RGB}{214,120,30}
\providecolor{pqOrangeText}{RGB}{170,74,12}
\providecolor{pqGreenFill}{RGB}{178,222,164}
\providecolor{pqGreenEdge}{RGB}{116,175,110}
\providecolor{pqGreenText}{RGB}{27,105,45}

\pgfplotsset{
  errplot/.style={
    xmin=0.7, xmax=4.3,
    xtick={1,2,3,4},
    xlabel={residual passes $P$},
    ylabel={relative error},
    ymin=4e-3, ymax=1,
    ytick={1e0,1e-1,1e-2},
    grid=major,
    grid style={gray!25},
    tick align=outside,
    label style={font=\scriptsize},
    tick label style={font=\scriptsize},
    ylabel style={yshift=-4pt},
    xlabel style={yshift=3pt},
  }
}

\usepackage{amsmath,amsfonts,bm}

\def\eqref#1{equation~\ref{#1}}

\def\1{\bm{1}}

\DeclareMathAlphabet{\mathsfit}{\encodingdefault}{\sfdefault}{m}{sl}
\SetMathAlphabet{\mathsfit}{bold}{\encodingdefault}{\sfdefault}{bx}{n}

\newcommand{\E}{\mathbb{E}}

\newcommand{\R}{\mathbb{R}}

\newcommand{\norm}[1]{\lVert #1 \rVert}
\newcommand{\Sph}{\mathbb{S}}
\newcommand{\AQuant}{ApexQuant}
\newcommand{\DecayLaw}{progressive decay law}
\newcommand{\DecayLawCap}{Progressive decay law}

\usepackage{hyperref}
\usepackage{url}
\usepackage{xcolor}
\usepackage{amsmath,amssymb}
\usepackage{booktabs}
\usepackage{tabularx}
\usepackage{amsthm}
\usepackage{mathtools}
\usepackage{cleveref}
\usepackage{multirow}

\newtheorem{proposition}{Proposition}
\newtheorem{corollary}{Corollary}
\newtheorem{lemma}{Lemma}

\title{ApexQuant: Data-Free Elastic Quantization \\ by Residual Re-Isotropization}

\author{Aksel Fristrup\thanks{Henselian (henselian.com)}\\
Department of Computer Science\\
University of Copenhagen\\
Denmark \\
\texttt{aksel@henselian.com}
\And{Sumit Pandey} \\
Towards Deep Learning\\
\texttt{sumitpandey171@gmail.com}
\And{Ankit Kariryaa} \\
Department of Computer Science\\
University of Copenhagen\\
Denmark \\
\texttt{ak@di.ku.dk} \\
}

\begin{document}
\pagestyle{plain}

\maketitle

\begin{abstract}

We introduce \AQuant{}, a calibration-free quantization method that recursively re-quantizes
the residual error, serving as a refinement layer on top of existing quantizers. We establish
that a fresh random rotation returns each residual to the uniform distribution on the
hypersphere, which characterizes the rate of progressive error decay across successive passes.
This result lets us determine, before any weight is read, how many passes a layer needs for a
target weight-space error. Every prefix is itself a valid lower-rate model, so one artifact
serves several precisions. We instantiate \AQuant{} with three interchangeable stages, scalar,
$E_8$ and trellis, and validate it on four open-weight LLMs and on Earth-observation and medical domains where in-distribution data is often
unattainable as imagery arrives under restrictive licences or due to patient material under privacy
constraints. Progressive re-isotropization comes within a few percent of full precision at
four bits and gives the best two-bit arm we measure, in a completely data-free setting.
\end{abstract}

\section{Introduction}
\label{sec:introduction}

\begin{figure}[t]
\centering
\includegraphics[width=0.98\linewidth]{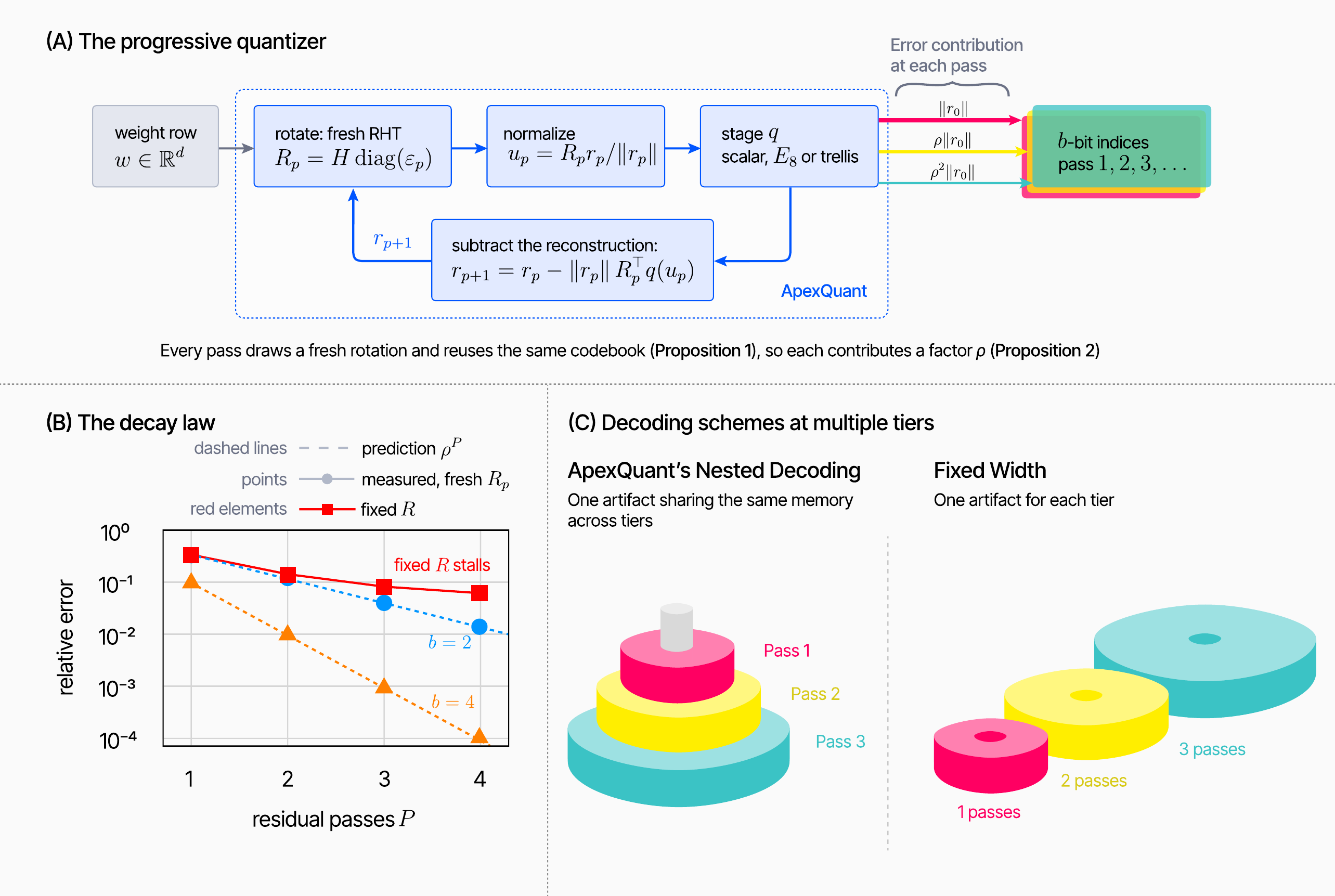}
\caption{\AQuant{}: the progressive quantizer end to end.
\textbf{(A)}~The three arrows into the index store are the error each pass contributes,
$1:\rho:\rho^{2}$, so later passes cost the same bits and correct geometrically less.
\textbf{(B)}~Dashed: the prediction $\rho^{P}$, a quadrature over $f_d$ as established in \cref{prop:decay}, with b=2 bits in blue, and b=4 bits in orange. Markers: difference measured on eight Clay v1.5 layers, agreeing to at least three decimals at
every pass.
Red: b = 2 bits with the rotation held \emph{fixed}, the counterfactual to
Proposition~\textcolor{pqLink}{1}.
\textbf{(C)}~A prefix of the passes is bit-for-bit a valid lower-rate model, so one artifact
serves every tier from one weight stream and a mixed batch reads each pass once. Storing one
artifact per tier instead shares nothing, which is the six blocks against ten: measured on a
Mistral block stack at tiers $(3,3,2,2)$, nesting stores $6.09$ against $10.16$ bits per
weight and decodes $1.56\times$ faster (\cref{sec:elasticmeasured}, \cref{app:serving}).}
\label{fig:overview}
\end{figure}

Large neural networks have become indispensable across scientific computing, software
development and natural-language applications, and their resource demands are growing at least
as fast as their capabilities \citep{SustainableAI}. Weight-only post-training
quantization (PTQ), which replaces high-precision weights with a compact low-precision
representation while requiring neither retraining nor architectural change, has therefore
become the dominant route to memory- and compute-efficient inference.

PTQ methods differ most consequentially in the data they require.
\emph{Calibration-dependent} methods fit quantization parameters to a held-out calibration
set; \emph{calibration-free} methods require no data at all, and can therefore quantize a
model the moment its weights exist. Calibration-free approaches avoid calibration-data bias entirely and remain applicable even when suitable calibration corpora cannot be assembled. Model hubs such as Hugging Face~\citep{fioravanti2025threemillion} now host over three million models spanning language, vision, audio and scientific domains; assembling a representative calibration set for these models is a challenging task and often infeasible due to legal restrictions. Earth-observation imagery, for instance, is often subject to licensing restrictions that prohibit redistribution, and sensitive clinical data remain confined to the institutions that collect them. A quantizer that requires no data sidesteps these obstacles entirely. Nonetheless, at low bit-widths calibration-free quantizers have
trailed calibration-dependent ones.

Rotation-based methods, which dominate the state of the art at low bit-widths, show that the
distribution can be fixed by construction. Multiplying the weights by a random
orthogonal transform, known as incoherence processing, suppresses outliers, after which the
rotated weights are quantized, either by adaptive rounding~\citep{ashkboos2024quarot} or codebook
lookup~\citep{tseng2024qtip, chee2023quip}. The consequence we exploit is that rotation
renders the source distribution known \emph{a priori}, so optimal codebooks can be designed
analytically, without observing any data. However, these codebooks are tied to a single bit-width chosen at quantization time. Several methods do support multiple bit-widths, including Any-Precision LLM~\citep{park2024anyprecision}, MoBiQuant~\citep{wang2026mobiquant} and AnyBCQ~\citep{park2026anybcq}, but all require calibration or training data. To our knowledge, no calibration-free elastic method exists.

In this paper we introduce \AQuant{}, a progressive PTQ method whose central contribution is
calibration-free elasticity, formalized by the \DecayLaw{} (\cref{prop:decay}). A single stored model exposes a
ladder of bit-widths to be chosen on demand. The method rests on the classical principle of
successive refinement of information~\citep{equitz1991successive}: each pass rotates the
current residual into a fresh frame, rounds it with the \emph{same} data-free
density-matched codebook, and subtracts the reconstruction, so $P$ passes yield a nested code
whose every prefix is itself a valid quantized model. We bring this elasticity to three stages,
the scalar Lloyd--Max codebook, QuIP\#'s $E_8$ lattice and QTIP's trellis. To make this elasticity hardware-efficient, we ship custom Triton kernels for the full pipeline: a register-resident Viterbi encoder that quantizes at $198$M weights per second ($68\times$ over a reference implementation), and fused dequantize-multiply kernels that serve mixed-precision batches from packed indices without reconstructing 16-bit weights. An optional single-batch calibration pass can further tighten the per-layer bit allocation with minimal data dependence (\cref{app:calib}). Code, kernels and reproduction scripts are available here: https://github.com/afristrup/ApexQuant.

First, we establish the optimality of progressive refinement at every pass. The enabling
result is re-isotropization (\cref{prop:reiso}): a fresh Haar rotation restores the
residual's spherical marginal $f_d$. Whichever codebook is optimal in its class is still
optimal at every pass, so nothing is ever re-fitted (\cref{cor:optimal}), and since the
argument never uses coordinate structure it covers the lattice and trellis stages too. Without the fresh rotation
the decay stalls, the residual retaining the structure the previous stage left behind
(\cref{fig:overview}B). Our theory is stated for Haar rotations; in practice, like QuIP\# and QTIP, we use the randomized Hadamard transform, which keeps the rotation $O(d\log d)$.
The gap between the two is at most two parts in a thousand on real weights (\cref{app:extra}).

Second, we derive the resulting \DecayLaw{} in closed form (\cref{prop:decay}):
$\mathrm{err}(P)=\rho^{P}$ in relative root-mean-square error, the constant from a quadrature
over a density that is fully determined prior to observing any weights. This prediction agrees with the error measured on pretrained weights to
three decimal places (\cref{fig:overview}B), and the same constant, without refitting, generalizes to non-language domains (\cref{sec:beyond}).

Third, we extend the theory to deployment. Because the two-bit representation \emph{is} the first pass of the four-bit one, a deployment that needs both precisions stores a single artifact rather than two independent files, and the model can be upgraded in place by appending indices. We demonstrate this nesting on Mistral-7B, where each transformer block is assigned one of two precisions (two or three passes, corresponding to approximately $4$ or $6$ bits per weight), yielding a mixed-precision configuration $(3,3,2,2)$ across blocks. Compared to the baseline of storing and loading one independent artifact per precision, the fused kernels decode this mixed-precision batch in $13.95$\,ms versus $21.76$\,ms ($1.56\times$ faster) and require $6.09$ versus $10.16$ bits per weight ($1.67\times$ less storage), because the shared prefix is read from memory only once (\cref{sec:elasticmeasured}). The two-bit prefix alone is not competitive as a standalone model; a calibrated trellis at that rate is far better (\cref{app:elastic}). Against five published data-free baselines at measured storage, \AQuant{} achieves the lowest KL at two bits on all four language models and, at four bits, the lowest among methods stored at our rate; against the published elastic methods it leads at four bits and trails a calibrated trellis at both rates (\cref{app:elastic}). As the method requires only weight access, it applies without modification to Earth observation and medical imaging (\cref{sec:beyond}).

\section{Background and related work}
\label{sec:background}

We consider three design axes for PTQ: the data requirement, the quantization dimension and
the elasticity. \Cref{tab:methods} places the methods closest to ours on all three.

\begin{table}[tb]
\centering
\caption{The methods closest to ours, on the three axes of this section. The data column is
what a method must read besides the weights, and the last is how many bit-widths one stored
model serves. \AQuant{} reads none, with the single optional exception of \cref{app:calib}.}
\label{tab:methods}
\footnotesize
\renewcommand{\arraystretch}{0.92}
\setlength{\tabcolsep}{5pt}
\begin{tabular}{@{}lllc@{}}
\toprule
Method & Data used & Quantizer & Bit-widths \\
\midrule
GPTQ, AWQ, OmniQuant & calibration set          & scalar            & one \\
AQLM                 & calibration set          & vector            & one \\
QuaRot               & calibration set          & scalar            & one \\
QuIP\#, QTIP         & calibration set (Hessian) & vector           & one \\
HQQ                  & none                     & scalar            & one \\
SINQ                 & none                     & scalar            & one \\
TurboQuant           & none                     & scalar            & one \\
Any-Precision LLM    & calibration set (seed)   & scalar            & several \\
MatQuant             & $10$--$100$M tokens      & scalar            & several \\
Drop-by-Drop         & $1024$ samples           & vector            & several \\
MoBiQuant            & $128$ sequences, $20$ epochs & scalar        & several \\
AnyBCQ               & $512$ sequences              & scalar  & several \\
\midrule
\AQuant{} (ours)     & none                     & scalar or vector  & several \\
\bottomrule
\end{tabular}
\end{table}

The majority of post-training quantization methods require calibration data: GPTQ, AWQ
and OmniQuant round each weight with parameters fitted on a calibration set
\citep{frantar2023gptq,lin2024awq,shao2024omniquant} and AQLM learns additive
vector codebooks from data \citep{egiazarian2024aqlm}. Other works removes this requirement, either through analytic bias correction \citep{nagel2019datafree} or through a synthesized calibration set
\citep{cai2020zeroq,xu2020gdfq}.

Among calibration based methods, the lowest reconstruction errors at low-bit results are achieved by rotation-based approaches, which transform the weight
matrix with a random orthogonal matrix, in practice a randomized Hadamard transform (RHT),
before quantizing it \citep{ashkboos2024quarot}. The rotation suppresses outliers by rendering it unlikely that any
single entry is large relative to the matrix's Frobenius norm, a property called incoherence
\citep{chee2023quip}, and QuIP\# showed the RHT delivers it with high probability at
$\Theta(d)$ sign flips and $\Theta(d\log d)$ additions \citep{tseng2024quipsharp}. The
rotation has a second, distributional consequence rooted in classical spherical geometry: under a Haar-random
rotation the transformed vector is uniform on the sphere, so each coordinate's marginal
density is known in closed form \citep{fang1990symmetric}.

\begin{lemma}[Spherical marginal]\label{lem:marginal}
    If $u$ is uniform on $\Sph^{d-1} = \bigl\{\mathbf{x} \in \mathbb{R}^d : \|\mathbf{x}\|_2 = 1\bigr\}$, each coordinate $u_i$ has density
    \begin{equation}
    f_d(t) \;=\; \frac{\Gamma\!\big(\tfrac{d}{2}\big)}
                      {\sqrt{\pi}\,\Gamma\!\big(\tfrac{d-1}{2}\big)}
                 \,\bigl(1-t^{2}\bigr)^{\frac{d-3}{2}},
    \qquad t\in[-1,1],
    \end{equation}
    so $(u_i+1)/2 \sim \mathrm{Beta}(\tfrac{d-1}{2},\tfrac{d-1}{2})$, $\E[u_i]=0$ and
    $\operatorname{Var}(u_i)=1/d$.
\end{lemma}

As $d$ grows this Beta distribution, rescaled by $\sqrt{d}$, converges to a Gaussian, a fact
exploited by \citet{tseng2024quipsharp} and \citet{tseng2024qtip} in their
codebook design. The second design axis is quantization dimension: encoding blocks of weights jointly rather than
individuals yields lower distortion at a given rate, which is the classical advantage of vector
quantization \citep{gersho1992vector}. QuIP\# encodes blocks of eight rotated
weights with the fixed $E_8$ lattice codebook \citep{tseng2024quipsharp}, and
QTIP encodes sequences of about $256$ weights with a trellis code whose codewords
are computed on the fly rather than stored in a table \citep{tseng2024qtip}.
Both operate at one fixed bit-width, so serving another means re-quantizing and
re-distributing the model.

Orthogonal to the rotation based-approaches, two data-free methods quantize unrotated weights and instead exploit structure that
the rotation deliberately removes. HQQ fits a per-group affine mapping by minimising an $l_p$
norm with $p{<}1$, whose sub-quadratic penalty on large residuals is better matched to the
heavy-tailed marginals of unrotated weights than a least-squares objective
\citep{badri2023hqq}; the method operates at group size $64$, which adds approximately $0.47$
bits per weight relative to $g{=}512$ but reduces the sensitivity of its min--max scale
estimate to outliers. SINQ estimates per-column scales by Sinkhorn-normalising the row and
column variances of $W$, obtaining a data-free approximation to the activation-aware column
scaling of AWQ \citep{muller2025sinq}; its calibrated variant A-SINQ replaces those
variance-derived scales with activation-fitted ones.
Both methods appear as baselines in
\cref{sec:experiments} at their native and at matched bit-budgets. Since the rotation that
provides us with a closed-form density also eliminates the tail shape and variance
heterogeneity on which HQQ and SINQ rely, the three approaches represent structurally distinct
strategies within data-free quantization.

Several methods support precision that can be selected after training, though all require calibration or training
data to do so. Any-Precision LLM constructs an 8-bit model whose leading bits already form the lower-bit
models \citep{park2024anyprecision}, MatQuant trains integer weights such that their leading
bits remain functional at 8, 4 and 2 bits\citep{nair2025matquant}, and Drop-by-Drop learns additive codebooks in which each stage
refines the reconstruction of the previous one \citep{babaoglu2026dropbydrop}. MoBiQuant quantizes the residuals into 1-bit slices and trains a router to select each token's precision at inference time \citep{wang2026mobiquant}. AnyBCQ extends binary-coded quantization to the multi-precision setting by progressively adding residual bit-planes whose binary codes are frozen at each level, so a single model supports $2$- through $4$-bit inference with a hardware-friendly bit-plane kernel \citep{park2026anybcq}.

The residual quantization mechanism we employ has a long history. Multi-stage vector quantization
originated in speech coding \citep{juang1982multistage}, and modern neural audio codecs are built on the same principle:
SoundStream and EnCodec quantize each residual with a further codebook, so a prefix of the
stages already decodes at a lower rate \citep{zeghidour2021soundstream,defossez2023encodec}.
The distinction from prior work is that those systems train their codebooks on
on the data being compressed and apply a single rotation fitted once and retained, whereas in our construction the rotation is
redrawn independently at every pass and fitted to no data, which enables \cref{prop:reiso} and yields a
closed form expression for the distortion curve .

\section{Re-isotropization and the \DecayLaw{}}
\label{sec:theory}

\begin{proposition}[Re-isotropization]\label{prop:reiso}
    Let $r\in\R^{d}\setminus\{0\}$ be a random vector and let $R$ be
    Haar-distributed on $O(d)$ (drawn uniformly at random from the group of
    $d\times d$ orthogonal matrices), independent of $r$. Then $u\coloneqq Rr/\norm{Rr}$
    is uniform on $\Sph^{d-1}$ and independent of $\norm{r}$; in particular each
    coordinate of $u$ has marginal $f_d$, regardless of the distribution of $r$.
\end{proposition}

Intuitively, a uniformly random rotation scrambles any input direction so that the
result points uniformly in all directions on the sphere, regardless of where it started.

Since $r$ may have any distribution, and in particular may be the original weight vector or the
residual of a previous pass, the statement applies at every pass (proof in \cref{app:proofs}). Setting $r_0=w$ and drawing mutually independent Haar rotations $R_0,R_1,\dots$, each
also independent of $w$, pass $p$ computes
\begin{equation}
\label{eq:pass}
    s_p = \norm{r_p}, \qquad
    u_p = R_p r_p / s_p, \qquad
    r_{p+1} = r_p - s_p\, R_p^{\top} q(u_p),
\end{equation}
where $q$ is the fixed stage quantizer, applied coordinate-wise in the scalar case and to
blocks otherwise. Each $r_{p+1}$ is a deterministic function of $(w,R_0,\dots,R_p)$, which
makes $R_{p+1}$ independent of it, so the proposition applies inductively at every pass.

In practice, as in QuaRot, QuIP\# and QTIP, we use the randomized Hadamard transform
$R=H\,\mathrm{diag}(\varepsilon)$, a finite family inside $O(d)$, so the orbit of a point under
it is finite and \cref{prop:reiso} is \emph{false} for it as stated. The gap is quantitative: on pretrained weights the deployed per-pass ratio departs from the Haar prediction by at most $1.6\%$ on individual layers at the first pass and by less than three parts in ten thousand at every later pass. Substituting a dense random-orthogonal rotation, which does satisfy \cref{prop:reiso}, moves the end-to-end error by at most two parts in a thousand at every $P$ (\cref{app:extra}).

\begin{corollary}[One codebook is optimal at every pass]\label{cor:optimal}
    Under~\eqref{eq:pass} the law (probability distribution) of $u_p$ is the same for every $p\ge 0$. Let
    $\mathcal{C}$ be any class of codebooks $C\subset\R^{d}$ (e.g., all $2^b$-level
    product codebooks, or all $E_8$ lattice codebooks) and write
    $D(C)=\E_u\operatorname{dist}(u,C)^{2}$ for the expected squared distance
    from $u$ uniform on $\Sph^{d-1}$ to its nearest codeword. If $C^{\star}$
    minimizes $D$ over $\mathcal{C}$, then nearest-codeword rounding to
    $C^{\star}$ minimizes $\E\norm{u_p-q(u_p)}^{2}$ at every residual pass among
    all quantizers with values in a member of $\mathcal{C}$, and no per-pass
    re-fitting of the codebook can improve on it in expectation. For
    $\mathcal{C}$ the $2^{b}$-level product codebooks, $C^{\star}$ is the
    Lloyd--Max codebook for $f_d$ in every coordinate; for blocks of eight
    $\mathcal{C}$ ranges over $E_8$ codebooks such as E8P, and for blocks of $T$
    over trellis codes.
\end{corollary}

Since the class $\mathcal{C}$ is unrestricted, each stage listed in \cref{tab:stages} is a
direct instantiation of the corollary: setting $\mathcal{C}$ to the set of $2^b$-level product
codebooks yields the Lloyd--Max codebook for $f_d$ as $C^{\star}$; setting it to $E_8$
codebooks or to trellis codes of a given constraint length yields the corresponding optimum.
In each case the optimal codebook is computed once, independently of the model weights, and
reused without modification at every pass.

Two qualifications apply, and neither of which concern the scale $s_p$, which is unconstrained and is re-estimated at every pass. First, the corollary compares codebooks \emph{within} a class and does not rank codebooks across
classes: it establishes
that the optimum within a class is invariant across passes, but the ordering between the
scalar, $E_8$ and trellis stages is determined by the empirical $\rho^{2}$ values reported in
\cref{tab:stages}. Second, the distortion measure is squared error on the normalized
residual direction, a weight-space proxy rather than end-to-end quality, which is why
\cref{sec:experiments} reports KL alongside $\rho$.

\begin{proposition}[\DecayLawCap{}]\label{prop:decay}
    Let $q$ be any measurable stage quantizer with stage distortion
    $\rho^{2}=\E_u\norm{u-q(u)}^{2}$, $u$ uniform on $\Sph^{d-1}$. Then
    under~\eqref{eq:pass},
    $\E\bigl[\norm{r_{p+1}}^{2}\bigr]=\rho^{2}\,\E\bigl[\norm{r_{p}}^{2}\bigr]$
    for every $p$, and hence $\sqrt{\E\norm{r_P}^{2}}\,/\,\norm{w}=\rho^{P}$.
    If $q$ applies one block quantizer $q_B$ to each of the $d/B$ blocks of $B$
    coordinates, then with $f_{d,B}$ the law of any $B$ coordinates of $u$,
    \begin{equation}
    \label{eq:rho}
    \rho^{2} \;=\;
    \frac{\E_{x\sim f_{d,B}}\norm{x-q_B(x)}^{2}}{\E_{x\sim f_{d,B}}\norm{x}^{2}}
    \;=\; \frac{d}{B}\,\E_{x\sim f_{d,B}}\norm{x-q_B(x)}^{2},
    \end{equation}
    which at $B=1$ reads $\rho^{2}=d\,\E_{t\sim f_d}[(t-q(t))^{2}]$.
\end{proposition}

Both integrals in \eqref{eq:rho} are quadratures over a density known in closed form
(\cref{lem:blockmarginal}), so the error curve and the pass rule
$P=\lceil\ln\varepsilon/\ln\rho\rceil$ are available before any weight is read.
The proof (\cref{app:proofs}) depends only on the fact that $R_p$ is an isometry and that
$u_p$ is independent of $s_p$ with a law that does not vary with $p$; it never uses the
coordinate structure of $q$. Consequently, a stage that does not factor over blocks, such as a
trellis code over $T$ coordinates, is accommodated by the definition of $\rho^2$ directly.
The quantity $\rho^{P}$ measures the gap between the achieved distortion and the
rate--distortion bound, incurred once per pass,
and is therefore a property of the stage
quantiser rather than of the recursion itself; it is reduced by improving the stage, as
reflected in the rate-loss column of \cref{tab:stages}.

\paragraph{Lean~4 check.}
\label{sec:lean}
All results in this section except \cref{prop:greedy}, which is arithmetic, is machine-checked
against Mathlib with no axiom of its own (\cref{lean:optimal}).

\begin{proposition}[The price of elasticity]
\label{prop:greedy}
Let a stage have per-pass ratio $\rho^{2}$ at $b$ bits per coordinate, and write
$\Delta_{1}=b+\tfrac{1}{2}\log_{2}\rho^{2}$ for its distance in bits to the Gaussian
rate--distortion function $D(R)=2^{-2R}$. Then $P$ passes store $Pb$ bits and reach
$\rho^{2P}$, which an optimal code reaches at $Pb-P\Delta_{1}$, so the greedy penalty is
exactly $P\Delta_{1}$.
\end{proposition}

Successive refinement is free for a Gaussian source under an optimal stage
\citep{equitz1991successive}; \cref{prop:greedy} shows that a suboptimal stage incurs
exactly $\Delta_{1}$ additional bits per pass, with the penalty accumulating additively rather
than multiplicatively. The rate-loss column of \cref{tab:stages} reports $\Delta_{1}$ for each
stage, and it is this quantity that improved stage design reduces, as visualized in \cref{fig:rateloss}. \Cref{cor:optimal}
establishes optimality within a codebook class, not relative to the rate--distortion function
$D(R)$; the property the construction preserves is that every pass operates in the same fixed
format.

\begin{figure}[t]
\centering
\includegraphics[width=0.85\linewidth]{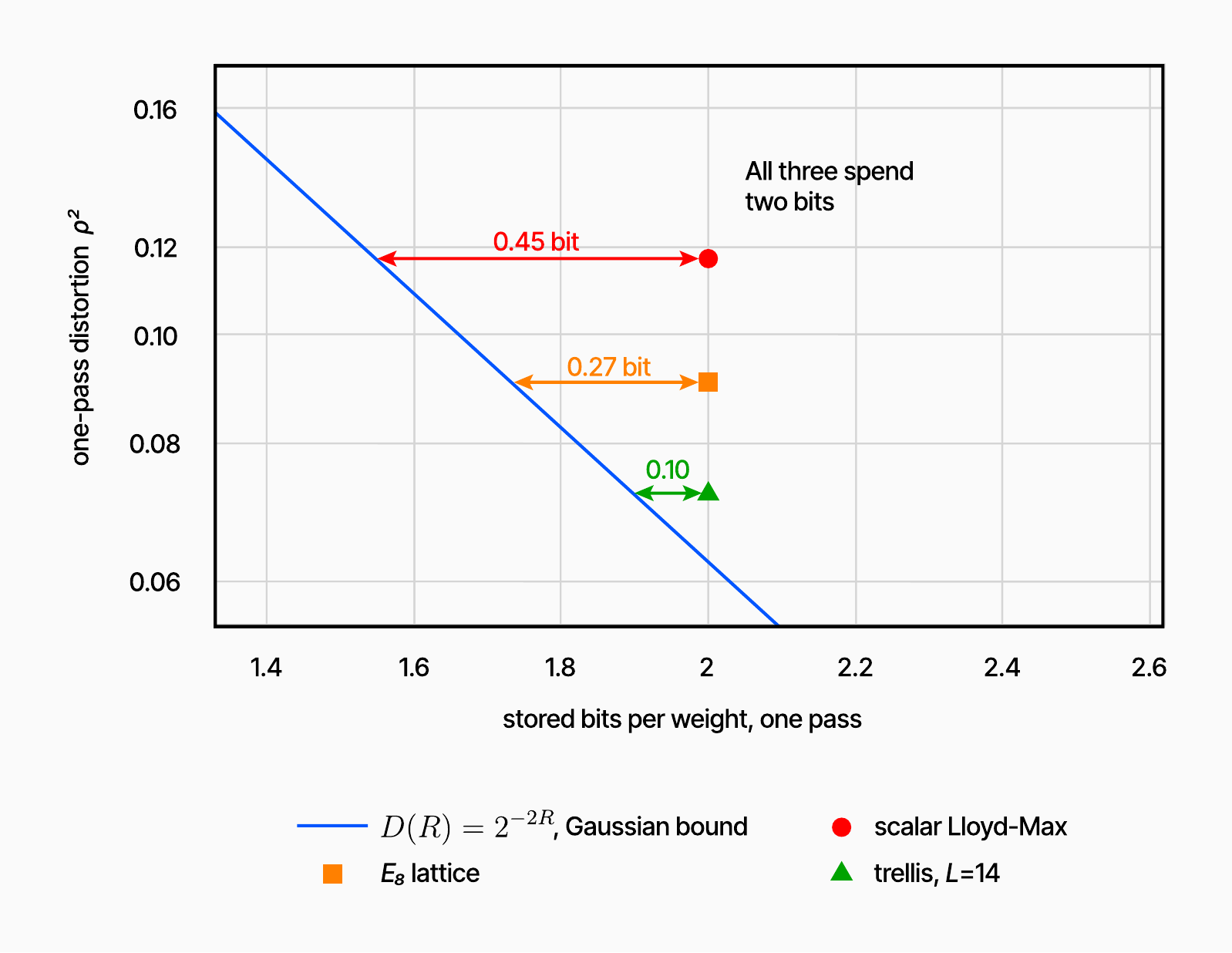}
\caption{The rate loss of \cref{tab:stages}, read off the rate--distortion bound. Each stage
sits above the Gaussian bound $D(R)=2^{-2R}$ by a horizontal distance measured in bits, and
strengthening the stage, from rounding one coordinate at a time to choosing $256$ together,
closes that distance from $0.45$ to $0.10$ bits. That distance is \cref{prop:greedy}'s
$\Delta_{1}$, and $P$ passes pay exactly $P$ of them: successive refinement is costless for a
Gaussian source, but only for a stage that achieves $D(R)$ exactly. Measured rather than argued
in \cref{app:llmtable}.}
\label{fig:rateloss}
\end{figure}

\section{\AQuant{}}
\label{sec:method}

\AQuant{} acts as a refinement layer on top of existing quantizers, in the spirit of
multistage quantization \citep{juang1982multistage}: because the decay law constrains the stage quantizer only through its distortion (\cref{prop:decay}), \eqref{eq:rho} fixes the error curve of whichever stage is dropped in. We instantiate the layer with three
data-free quantizers, called \emph{stages} as in the multistage literature,
in increasing order of strength;
\cref{tab:stages} collects their $\rho^{2}$ constants, and \cref{fig:stages} visualizes the encoding flow.

\subsection{Three data-free stages}
\label{sec:stages}

\emph{Scalar.} The scalar stage rounds each coordinate of $u$ to the
nearest value of the Lloyd--Max codebook for $f_d$ from
\cref{sec:background}, giving $\rho^{2}=0.1171$ at $b=2$ and
$d=512$, which is within one percent
of the 0.118 that TurboQuant published~\cite{zandieh2025turboquant}. The codebook is the same as in TurboQuant, which proposed it for KV cache quantization, and is determined entirely by $f_d$. Since quantization requires no optimization, the quantized model is available immediately upon encoding.

\emph{$E_8$ lattice.}
The lattice stage rounds eight coordinates at a time to the nearest point of
the $E_8$ lattice, the densest sphere packing in eight dimensions
\citep{viazovska2017sphere}, using the compact codebook of QuIP\#
\citep{tseng2024quipsharp}. On rotated unit vectors we measure
$\rho^{2}=0.0906$, within two percent of the $0.089$ that QuIP\# publishes. At three or more bits, QuIP\# also
recurses on the residual but without a fresh rotation between rounds, which is the key
ingredient that enables the data-free elasticity of \cref{sec:elastic} and without which the decay stall. As with the scalar stage, the quantized model is available immediately upon encoding.

\begin{table}[tb]
\centering
\caption{The three stages evaluated at two bits per weight. $\rho^{2}$ is the single-pass error of
\eqref{eq:rho}, the published column reports the values given by QTIP and QuIP\# for a Gaussian source; and rate loss denotes the additional bits per pass relative to the final row, computed as $\tfrac{1}{2}\log_2(\rho^{2}/
\rho^{2}_{\mathrm{RD}})$ from our measured column. The trellis row uses $L=14$, the widest our
kernel accelerates (\cref{app:kernels}).}
\label{tab:stages}
\small
\setlength{\tabcolsep}{5pt}
\begin{tabular}{lccccl}
\toprule
Stage & Dim. & $\rho^{2}$ (ours) & $\rho^{2}$ (publ.) & Rate loss (bits/pass) & Encoding \\
\midrule
Scalar Lloyd--Max & 1 & 0.1171 & 0.118 & 0.45 & nearest of $2^{b}$ values \\
$E_8$ lattice & 8 & 0.0906 & 0.089 & 0.27 & nearest lattice point \\
Trellis ($L=14$) & 256 & 0.0720 & 0.069 & 0.10 & Viterbi search \\
\midrule
Rate--distortion bound & $\infty$ & --- & 0.0625 & 0 & --- \\
\bottomrule
\end{tabular}
\end{table}

\emph{Trellis.} Trellis coded quantization \citep{marcellin1990trellis} encodes $T$
coordinates as a walk on a $2^{L}$-state graph, the minimum-error walk being found by the
Viterbi algorithm \citep{forney2005viterbi} in $O(2^{L}T)$ time rather than the $2^{bT}$ of an
unstructured codebook. We adopt QTIP's bitshift trellis with lookup-free computed codes
\citep{tseng2024qtip}, which stores neither graph nor codebook and so is as data-free as the
other two stages, and we replace its Hessian-guided rounding step with our residual passes.
On rotated unit vectors we measure $\rho^{2}=0.0720$ at $L=14$, a configuration reached by the
Triton Viterbi encoder we contribute (\cref{app:kernels}).

The rate-loss column of \cref{tab:stages} is \cref{prop:greedy}'s $\Delta_{1}$, against the
Gaussian bound $\rho^{2}=2^{-4}=0.0625$ at two bits \citep{cover1991elements}. It falls with the
dimension quantized jointly, $0.45$ bits for the scalar stage, $0.27$ for $E_8$, $0.10$ for the
trellis.

\subsection{Elastic inference}
\label{sec:elastic}

Since each pass quantizes only the residual and never changes the output of
previous passes, a model quantized with $P$ passes can be read at any
prefix $P'\le P$ and the result is bit-identical to a run stopped at $P'$. We call this
\emph{nested decoding}. Reading the first $P'$ passes of the stored payload is not an
approximation of the $P'$-pass artifact, it \emph{is} that artifact, checked by byte comparison in
\cref{sec:elasticmeasured}. Therefore, one
stored file serves several precisions, a deployed model upgrades to $P+1$ passes in
place by appending indices, and a batch may mix quality tiers at the cost of its deepest tier
alone.

That last property is where nesting provides its principal advantage, as it nests the \emph{execution} path and not
merely the stored representation. Rows are sorted by descending precision tier, so that the
rows consuming pass~$p$ form a contiguous prefix; the matrix multiplication
then proceeds through passes in order, with pass one read once for the
entire batch and pass two read only for rows requiring two or more passes. The resulting weight traffic equals that of the highest-precision tier
alone, rather than the sum across individual rows. By contrast, fixed-width
elasticity, i.e. one artifact per tier, shares no data between tiers and
incurs precisely that sum (\cref{sec:elasticmeasured},
\cref{app:serving}, which also distinguishes this mechanism from
speculative decoding).

At inference, the stored representation is never decompressed into 16-bit weights: since
$x^{\top}\widehat{w} = s\,(Rx)^{\top} q(u)$, the fused kernels rotate the activation once per
pass and multiply directly against codebook values gathered from the stored $Pb$-bit indices.

\section{Experiments}
\label{sec:experiments}

\emph{Setup.} We evaluate TinyLlama-1.1B \citep{zhang2024tinyllama}, Phi-1.5
\citep{li2023phi15}, Mistral-7B \citep{jiang2023mistral} and Qwen2.5-7B
\citep{qwen2024qwen25} on WikiText-2 raw at sequence length $2048$ \citep{merity2017pointer},
and beyond language (\cref{sec:beyond}), the Earth observation foundation model Clay v1.5 \citep{clay2024}
on EuroSAT \citep{helber2019eurosat} and the pathology foundation model UNI \citep{chen2024uni} on PatchCamelyon \citep{veeling2018patchcamelyon}.
Our primary metric is
$\mathrm{KL}(p_{\mathrm{orig}}\,\|\,p_{\mathrm{quant}})$ per token against the unquantized
model, following the model-preservation view of \citet{dutta2024accuracy}.

\begin{table}[tb]
\centering
\caption{Language models, $\mathrm{KL}(p_{\mathrm{fp16}}\,\|\,p_{\mathrm{quant}})$ in nats per
token on WikiText-2 raw at sequence length $2048$, lower is better.
Bits per weight is \emph{measured} storage including side
information, which puts the group-size-$64$ arms of HQQ and SINQ at $4.50$ and $4.53$ against
our $4.06$; both also appear at $g=512$, matched to ours. SINQ's overhead is shape-dependent,
so its entries are bounds.
In the $4$-bit class the
scalar stage is one pass at four bits and the vector stages two passes at two, the only $4$-bit
form here keeping a valid $2$-bit prefix. Every variant here is data-free; and \cref{app:llmtable} includes the extended results with the calibrated
counterpart A-SINQ as well as perplexity, tail
percentiles, the $l_p$ codebook and a group-size sweep. Model abbreviations: TL = TinyLlama-1.1B, Phi=Phi-1.5, Mist.=Mistral-7B, and Qwen = Qwen2.5-7B.
}
\label{tab:llm-full}
\scriptsize
\setlength{\tabcolsep}{4pt}
\renewcommand{\arraystretch}{0.92}
\begin{tabular}{@{}l|rrrrr|rrrrr@{}}
\toprule
& \multicolumn{5}{c|}{$4$-bit class} & \multicolumn{5}{c}{$2$-bit class} \\
Method & b/w & TL & Phi & Mist. & Qwen & b/w & TL & Phi & Mist. & Qwen \\
\midrule
RTN-absmax               & $4.03$ & $4.419$ & $0.130$ & $0.247$ & $0.148$ & $2.03$ & $8.056$ & $6.742$ & $8.939$ & $14.532$ \\
QuaRot-RTN               & $4.03$ & $0.116$ & $0.059$ & $0.054$ & $0.072$ & $2.03$ & $8.144$ & $2.552$ & $8.543$ & $4.217$ \\
SINQ, $g=64$             & $\leq\!4.53$ & $\mathbf{0.035}$ & $0.028$ & $\mathbf{0.019}$ & $0.039$ & $\leq\!2.53$ & $4.537$ & $1.431$ & $2.107$ & $6.152$ \\
SINQ, $g=512$            & $\leq\!4.09$ & $0.069$ & $0.050$ & $0.040$ & $0.096$ & $\leq\!2.09$ & $10.230$ & $4.476$ & $8.887$ & $9.985$ \\
HQQ, $g=64$              & $4.50$ & $0.046$ & $0.033$ & $0.025$ & $0.044$ & $2.50$ & $8.115$ & $1.912$ & $6.361$ & $4.091$ \\
HQQ, $g=512$             & $4.06$ & $0.095$ & $0.068$ & $0.050$ & $0.096$ & $2.06$ & $8.223$ & $6.731$ & $8.215$ & $12.294$ \\
\midrule
\AQuant{}, scalar        & $4.03$ & $0.074$ & $0.040$ & $0.037$ & $0.049$ & $2.03$ & $7.188$ & $1.160$ & $4.392$ & $1.334$ \\
\AQuant{}, $E_8$P        & $4.06$ & $0.062$ & $0.035$ & $0.030$ & $0.043$ & $2.03$ & $5.756$ & $0.693$ & $2.836$ & $0.681$ \\
\AQuant{}, trellis       & $4.06$ & $0.040$ & $\mathbf{0.022}$ & $\mathbf{0.019}$ & $\mathbf{0.027}$ & $2.03$ & $\mathbf{3.201}$ & $\mathbf{0.447}$ & $\mathbf{1.903}$ & $\mathbf{0.477}$ \\
\bottomrule
\end{tabular}
\end{table}

\Cref{tab:llm-full} compares \AQuant{} against five published baselines at comparable bit use.
The complexity ladder of \cref{tab:stages} is preserved throughout: trading the scalar stage for the trellis
is worth $1.8$ to $1.9\times$ in KL at four bits and $2.2$ to $2.8\times$ at two. The same ordering holds end to end on Gemma-3-12B (\cref{app:stageorder}). \AQuant{}, trellis performs best in the 2-bit class. In the 4-bit class, SINQ at g=64 matches it on Mistral and surpasses it on TinyLlama, albeit at half a bit more budget. At comparable bit budget, \AQuant{}, trellis outperforms all baselines. Classification accuracy on standard benchmarks confirms this ordering (\cref{tab:zeroshot}).

\subsection{The \DecayLaw{} holds on real weights}
\label{sec:decaymeasured}

The prediction of \cref{prop:decay}, derived without access to any model weights, agrees with the empirical values on pretrained weights to at least three decimal places at every pass.
Over eight Clay v1.5 layers at $d=512$ the measured relative error runs $0.3419$,
$0.1170$, $0.0400$, $0.0137$ across one to four passes at $b=2$ against a predicted
$0.3422^{P}$. At $b=4$ it runs $0.0973$, $0.0095$, $0.0009$, $0.0001$ against $0.0972^{P}$
(\cref{fig:overview}B).
The per-pass ratios remain constant at $0.342$ and
$0.097$, consistent with exact geometric decay rather than an empirically fitted trend.

In case the rotation is fixed across passes, the ratios climb $0.418$, $0.597$, $0.734$ towards one, thus the decay stalls once the residual is
no longer returned to $f_d$. Replacing the Hadamard transform with a dense random-orthogonal rotation, for which
\cref{prop:reiso} is exact, moves the error by at most two parts in a thousand at every pass
count (\cref{app:extra}). Perplexity shows the same order but by far smaller
margins, which establishes the importance of KL as the primary metric (\cref{app:llmtable}).

\subsection{Elasticity in operation}
\label{sec:elasticmeasured}

Next, we show that elasticity comes at no cost to performance. Reading the first $P'$
passes of a $P$-pass payload is bit-identical to quantizing with $P'$, checked by byte
comparison on three published tiered models, thus raising the precision is an append operation and not a
requantization. On Mistral-7B with $32$ blocks assigned to two precisions in a $(3,3,2,2)$ configuration, our fused kernels decode a mixed-precision batch in $13.95$\,ms versus $21.76$\,ms when each precision is served from an independent artifact ($1.56\times$ faster), and require $6.09$ versus $10.16$ bits per weight ($1.67\times$ less storage).
\Cref{app:serving} reports five serving scenarios, including the memory-bound regime where nesting provides no latency gain: inside that window every request costs the same regardless of precision, so the latency advantage requires a deep batch; the storage reduction holds unconditionally.

\subsection{Beyond language}
\label{sec:beyond}

Our construction is domain-agnostic as it only requires the weight metrics.
\Cref{tab:beyond} demonstrates this generality on foundation models from Earth observation and pathology. The tiers behave as predicted by the theory. Both encoders' KL falls by roughly an
order of magnitude per tier, $0.641$ to $0.033$ to $0.003$ on UNI and $0.224$ to $0.010$ to
$0.002$ on Clay. The geometric
decay of \cref{prop:decay} is thus observable in a downstream model-preservation metric, not
solely in the weight-space $\rho$.

\begin{table}[t]
\centering
\caption{Results beyond language modeling, evaluated at the three precision tiers ($2.06$, $4.13$ and $6.19$ bits).
KL divergence is against the same model at full precision.
The Clay probe carries a standard error of $0.43$ points, so
only KL divergence separates its tiers. The KL protocol, the sampling procedure and each model's quantized fraction are available in
\cref{app:nonllm}.}
\label{tab:beyond}
\scriptsize
\setlength{\tabcolsep}{4pt}
\begin{tabular}{@{}llr|rrr|rrr@{}}
\toprule
& & & \multicolumn{3}{c|}{task metric ($\uparrow$)} & \multicolumn{3}{c}{KL ($\downarrow$)} \\
Model & Task & full & $P{=}1$ & $P{=}2$ & $P{=}3$ & $P{=}1$ & $P{=}2$ & $P{=}3$ \\
\midrule
Clay v1.5 & EuroSAT      & $96.30$ & $96.00$ & $96.15$ & $96.35$ & $0.224$ & $0.010$ & $0.002$ \\
UNI       & PatchCamelyon& $87.81$ & $85.34$ & $86.68$ & $87.56$ & $0.641$ & $0.033$ & $0.003$ \\
\bottomrule
\end{tabular}
\end{table}

\section{Limitations}
\label{sec:limitations}

Calibration data allows a quantizer to learn which weight perturbations are costly to the model's output and steer rounding error away from them. \AQuant{} treats all directions equally, and this is its dominant cost: on LLaMA2-7B a calibrated trellis (QTIP) reaches $5.17$ perplexity versus our $5.61$ at four bits and $5.86$ versus $71.80$ at two (\cref{app:elastic}). \Cref{prop:reiso} is proved for Haar rotations; the deployed Hadamard transform does not satisfy its hypotheses, and the empirical gap of under two parts per thousand (\cref{app:extra}) is a measurement, not a bound. The largest model we evaluate end to end is 7B parameters; the decay law has not been verified on a full forward pass at frontier scale.

\section{Conclusion}
\label{sec:conclusion}

In this paper we introduce \AQuant{}, a calibration-free quantizier that recursively re-quantizes the residual. We establish that a fresh Haar rotation returns each residual to the same spherical marginal, so one codebook stays optimal at every pass, the error curve is available in the closed form  before a weight is read, and every prefix is a valid lower-rate artifact. A model can therefore
be quantized the moment its weights exist and the quantization is free from any calibration set bias. In addition, the quantization is near-instant with the scalar, and E8 stages, and we introduce a fast Triton Viterbi encoder quantizes the models up to 68 times faster than the baseline. The resulting models outperform all calibration-free baselines at matched bit-budgets and substantially narrow the gap to their calibrated counterparts.

\subsection*{AI use statement}

In this work, we used generative AI tools (such as Claude Opus 4.6, Claude Opus 5, Claude Fable 5 and others) for code development (Python/Triton kernels and Lean~4 formalization), drafting of parts of some sections based on author-provided structure, editing of text, and discovery of recent literature in the quantization domain. Additionally, we used generative AI tools for assistance with figure layout and styling. All proofs were done by hand originally but refined and verified with lean4 and LLMs in the loop. All generated code, and textual edits were manually reviewed by the authors.  We take responsibility for the final content of this work, including text, claims or artifacts produced with the aid of generative AI.

\bibliography{references}
\bibliographystyle{references}

\appendix

\section{Proofs}
\label{app:proofs}

\paragraph{Notation.} We reintroduce the symbols here for clarity.
$\rho^{2}$ is the \emph{stage} distortion of
\cref{prop:decay}, $\E_u\norm{u-q(u)}^{2}$ for $u$ uniform on $\Sph^{d-1}$: a property of the
codebook alone, computed by quadrature before any weight exists, and tabulated in
\cref{tab:stages}. $\rho$ is defined as its positive square root. The per-pass ratio
measured on real weights is an estimate of that same $\rho^{2}$.
Finally $\mathrm{err}(P)$ is a
\emph{relative root-mean-square} error, $\sqrt{\E\norm{r_P}^{2}}/\norm{w}$, which is why
$\mathrm{err}(P)=\rho^{P}$ and \cref{prop:decay}'s
$\E\norm{r_{P}}^{2}=\rho^{2P}\,\E\norm{w}^{2}$ are the same statement. Bit-widths follow the
same split: $b$ is the stage's bits per coordinate, $Pb$ the payload, and the bits per weight
quoted in every table are \emph{measured} storage including side information, so they exceed
$Pb$ by the scale overhead.

\begin{proof}[Proof of Proposition~\ref{prop:reiso}]
We condition on $r$. For fixed $r\neq 0$ and any $Q\in O(d)$, Haar measure is
invariant under left multiplication, so $QR$ has the same law (probability distribution)  as $R$ and hence
$Q\,(Rr/\norm{Rr}) = (QR)r/\norm{(QR)r}$ has the same law as $Rr/\norm{Rr}$. The
conditional law of $u=Rr/\norm{Rr}$ on $\Sph^{d-1}$ is therefore invariant under
the action of $O(d)$, and the unique such probability measure is the uniform one
(since any distribution on the sphere that is unchanged by every rotation cannot
favor any region over another of equal area).
The conditional law does not depend on $r$, so it is also the unconditional law
and $u$ is independent of $r$, in particular of $\norm{r}$, using $\norm{Rr} =
\norm{r}$. The coordinate marginal of the uniform distribution on $\Sph^{d-1}$
is $f_d$ by Lemma~\ref{lem:marginal}.
\end{proof}

\begin{proof}[Proof of Proposition~\ref{prop:decay}]
Fix $p$ and abbreviate $s=s_p$, $u=u_p$, $R=R_p$. Since $R$ is an isometry
($\norm{Rx}=\norm{x}$ for all $x$, so multiplying by $R$ inside a norm does not
change it),
\[
\norm{r_{p+1}}^{2}
= \norm{r_p - s\,R^{\top}q(u)}^{2}
= \norm{R r_p - s\,q(u)}^{2}
= s^{2}\,\norm{u - q(u)}^{2},
\]
since $Rr_p = s\,u$ by definition of $u_p$. By Proposition~\ref{prop:reiso} the direction $u$ is uniform on
$\Sph^{d-1}$ and independent of $s=\norm{r_p}$, so
\[
\E\bigl[\norm{r_{p+1}}^{2}\bigr]
= \E[s^{2}]\;\E\bigl[\norm{u-q(u)}^{2}\bigr]
= \E\bigl[\norm{r_{p}}^{2}\bigr]\; d\,\E_{t\sim f_d}\bigl[(t-q(t))^{2}\bigr],
\]
where the last equality decomposes $\norm{u-q(u)}^{2}$ into $d$ coordinate terms, each
with marginal $f_d$, and using $\E_{t\sim f_d}[t^{2}] = 1/d$. The bracketed
factor equals $\rho^{2}$ as defined in \eqref{eq:rho} and does not depend on $p$. Induction from
$r_0=w$ gives $\E\norm{r_P}^{2} = \rho^{2P}\norm{w}^{2}$, hence the stated form.
Only the isometry of $R$, the independence of $u$ from $s$ and the constancy of
the law of $u$ in $p$ were used; the coordinate structure of $q$ was not, so
replacing the quadrature by $\rho^{2}=\E_u\norm{u-q(u)}^{2}$ carries the
statement to a lattice or trellis stage verbatim.
\end{proof}

\begin{lemma}[Block marginal]\label{lem:blockmarginal}
    Let $u$ be uniform on $\Sph^{d-1}$ and let $x\in\R^{B}$ collect any $B<d$ of its
    coordinates. Then $x$ has density
    \begin{equation}
    \label{eq:blockmarginal}
        f_{d,B}(x) \;=\;
        \frac{\Gamma(d/2)}{\pi^{B/2}\,\Gamma\!\bigl((d-B)/2\bigr)}
        \,\bigl(1-\norm{x}^{2}\bigr)^{\frac{d-B-2}{2}}
    \end{equation}
    on the unit ball of $\R^{B}$, and $B=1$ recovers \cref{lem:marginal}.
\end{lemma}

This is the law the $E_8$ and trellis stages integrate against in \eqref{eq:rho}, so the
quadrature that fixes $\rho^{2}$ is available in closed form at every block size and not
only in the scalar case. Both the density and its normalising constant are machine-checked.

\paragraph{What the Lean~4 development machine-checks.}
\label{lean:optimal}
The development comprises $154$ public theorems against Mathlib. It introduces no
\texttt{axiom} of its own and contains no \texttt{sorry}: \texttt{\#print axioms} on every
public result returns \texttt{propext}, \texttt{Classical.choice} and \texttt{Quot.sound}
and nothing else. Each result below is either proved outright or carries its hypotheses in
its own signature.

\Cref{prop:reiso} is checked at the concrete setting of this paper, not only in the abstract.
$O(d)$ is built as \texttt{Matrix.orthogonalGroup}, shown compact, equipped with its
normalised Haar measure, shown right-invariant, acting on $\R^{d}$ by isometries and
transitively on the unit sphere; \texttt{Orthogonal.orbitLaw\_\allowbreak haar} is then
\cref{prop:reiso} in full. The uniqueness step the informal proof invokes is proved
separately: the uniform law is the \emph{only} rotation-invariant probability measure
carried by the sphere (\texttt{Orthogonal.eq\_\allowbreak sphereUniformAmbient\_\allowbreak of\_\allowbreak invariant}).
\Cref{lem:marginal} is proved as a distributional statement about Mathlib's normalised
surface measure, together with its block form \eqref{eq:blockmarginal} and that form's
constant (\texttt{sphereUniform\_\allowbreak map\_\allowbreak block}, \texttt{blockLaw\_\allowbreak eq\_\allowbreak withDensity}).

\Cref{cor:optimal} is checked in the class-indexed form of \cref{sec:theory}, for an
arbitrary class of codebooks in an arbitrary real normed measurable space: rounding to the
class-optimal codebook beats every stage taking values in that class
(\texttt{stageDistortion\_\allowbreak le\_\allowbreak of\_\allowbreak codebook\_\allowbreak optimal}), and the statement is carried through
to the residual after $P$ passes
(\texttt{integral\_\allowbreak norm\_\allowbreak resid\_\allowbreak sq\_\allowbreak le\_\allowbreak of\_\allowbreak codebook\_\allowbreak optimal}). That per-block nearest
rounding is nearest in the product codebook is checked too
(\texttt{isNearest\_\allowbreak blockStage}), which is what makes coordinate-wise Lloyd--Max and
per-block E8P instances rather than analogies.

\Cref{prop:decay} is not assumed isotropic per pass; the isotropy is derived. Its
hypotheses are bundled, and every field of that bundle is discharged from hypotheses that
speak only about the algorithm's inputs, namely that the rotations are mutually independent
of each other and of the weights, each drawn from one right-invariant law, that the acting
group is transitive on the unit sphere, and that no pass is exact
(\texttt{isotropicPasses\_\allowbreak of\_\allowbreak rotationStream}). Two consequences follow that a conditional
form could not state. The residual after $p$ passes is a measurable function of the weight
row and the earlier rotations, so the fresh rotation is independent of everything the
residual depends on, which is the induction the prose argues. And $\rho$ ceases to be a
free parameter: it is \emph{identified} as the stage distortion of the codebook against the
direction law, which is what ties \cref{prop:decay}'s constant to the stage-agnostic
quantity of \cref{tab:stages}. The hypotheses are then shown satisfiable at $O(d)$ for
\emph{every} $d\ge1$
(\texttt{exists\_\allowbreak rotationStream\_\allowbreak Od}, \texttt{exists\_\allowbreak witness\_\allowbreak decay\_\allowbreak Od}), and
\texttt{integral\_\allowbreak norm\_\allowbreak resid\_\allowbreak sq\_\allowbreak paper} is precisely \cref{prop:decay} as stated in \cref{sec:theory}:
Haar rotations on $O(d)$, a scalar codebook applied coordinate-wise, and $\rho^{2}$ the
quadrature of \eqref{eq:rho}.

\section{Haar in theory, Hadamard in practice}
\label{app:extra}
\begin{figure}[t]
\centering
\includegraphics[width=\linewidth]{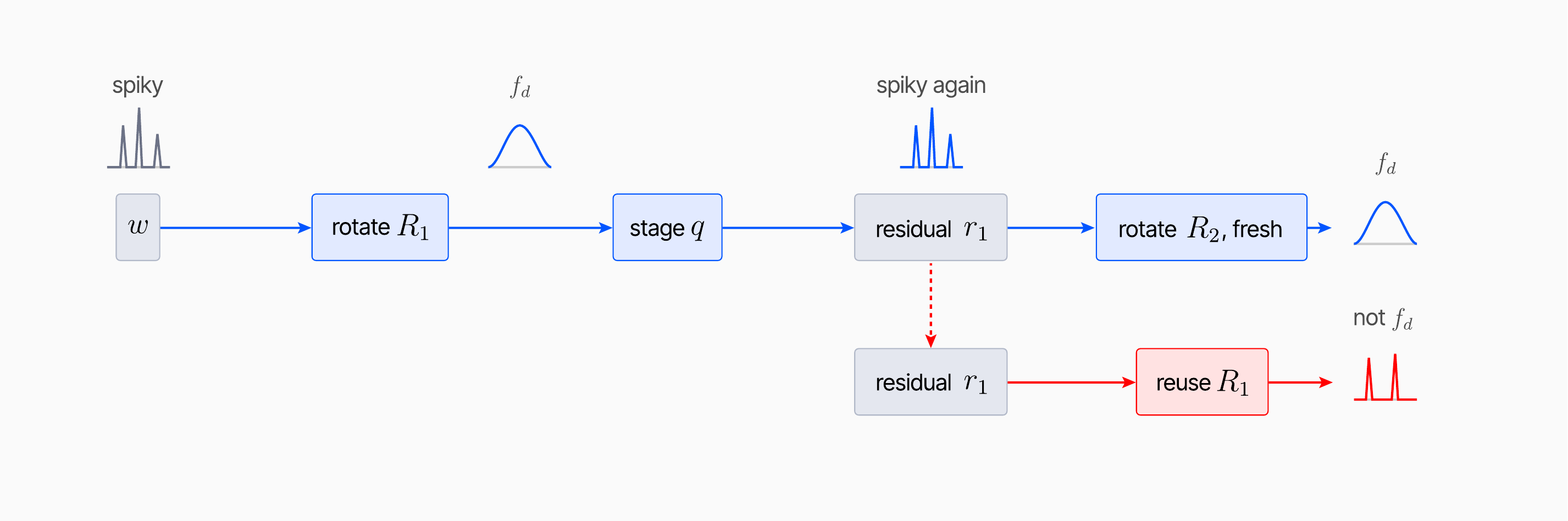}
\caption{Why the rotation has to be fresh? A stage leaves a residual that is structured in the
frame it was just quantized in, so the marginal the codebook was designed for is destroyed by
the very act of using it. Drawing a fresh rotation restores that marginal exactly
(\cref{prop:reiso}), which is what lets the same codebook be optimal at pass two as at pass
one, so nothing is refitted (\cref{cor:optimal}) and the error is $\rho^{P}$
(\cref{prop:decay}). Reusing the previous rotation does not: the residual is already adapted
to it, the coordinates it produces are not $f_d$, and the geometric decay stalls, its measured
per-pass ratios climbing $0.418$, $0.597$, $0.734$ towards one (\cref{fig:overview}B).
Freshness, not orthogonality, is the active ingredient. The coordinate laws are drawn
schematically.}
\label{fig:reiso}
\end{figure}

\Cref{prop:reiso} is exact for Haar rotations. The deployed rotation is
$R=H\,\mathrm{diag}(\varepsilon)$ with $H$ flat and orthogonal and $\varepsilon$ i.i.d.\
signs, a finite subset of $O(d)$ of size $2^{d}$, so the orbit of any point under it is
finite and the hypotheses of \cref{prop:reiso} are not satisfied. The resulting discrepancy
is direction-dependent and vanishes after a single pass on real weights (\cref{fig:reiso}). We quantify the end-to-end cost of this
substitution below. Note that the difference here is limited to weights and may not hold for embeddings and activations.

\paragraph{Substituting a true Haar rotation.}
To isolate the effect of the discrete rotation we replace it with a dense
random-orthogonal matrix, which satisfies \cref{prop:reiso} exactly, and compare the
reconstruction error end to end. On eight Clay v1.5 encoder layers, at $d=512$ and $b=2$
with a fresh rotation per pass:

\begin{center}
\small
\begin{tabular}{@{}rrrr@{}}
\toprule
$P$ & Hadamard & dense random-orthogonal & relative gap \\
\midrule
$1$ & $0.3419$ & $0.3424$ & $1.4\times10^{-3}$ \\
$2$ & $0.1170$ & $0.1173$ & $2.0\times10^{-3}$ \\
$3$ & $0.0401$ & $0.0401$ & $1.8\times10^{-3}$ \\
$4$ & $0.0137$ & $0.0137$ & $8.5\times10^{-4}$ \\
\bottomrule
\end{tabular}
\end{center}

The substitution changes the error by at most two parts in a thousand, and the gap does not
grow with the pass count. The same measurement on Phi-1.5's decoder linears at both
deployed group sizes confirms that the conclusion is not specific to the vision encoder:

\begin{center}
\small
\begin{tabular}{@{}rrrrrrr@{}}
\toprule
& \multicolumn{3}{c}{$d=512$} & \multicolumn{3}{c}{$d=256$} \\
\cmidrule(lr){2-4}\cmidrule(lr){5-7}
$P$ & Hadamard & dense r.o. & gap & Hadamard & dense r.o. & gap \\
\midrule
$1$ & $0.3419$ & $0.3422$ & $9.2\times10^{-4}$ & $0.3412$ & $0.3417$ & $1.4\times10^{-3}$ \\
$2$ & $0.1170$ & $0.1171$ & $1.3\times10^{-3}$ & $0.1166$ & $0.1167$ & $1.5\times10^{-3}$ \\
$3$ & $0.0400$ & $0.0401$ & $1.4\times10^{-3}$ & $0.0398$ & $0.0399$ & $1.7\times10^{-3}$ \\
$4$ & $0.0137$ & $0.0137$ & $1.2\times10^{-3}$ & $0.0136$ & $0.0136$ & $2.0\times10^{-3}$ \\
\bottomrule
\end{tabular}
\end{center}

The substitution gap stays at or below two parts in a thousand at every pass count and at
both group sizes. At $d=256$ it grows mildly with $P$, from $1.4$ to $2.0$ parts in a
thousand, consistent with a smaller group affording less averaging, but well below the
level at which the quantized model would be affected. The deployed Hadamard column is
itself unchanged across the two model families, $0.3419$, $0.1170$, $0.0400$ and $0.0137$
on Phi-1.5 against Clay's $0.3419$, $0.1170$, $0.0401$ and $0.0137$, agreeing to every
digit but the fourth of the third pass, confirming that \cref{prop:decay}'s constant is
independent of the weight distribution to the precision we measure
(\texttt{experiments/haar\_substitution\_llm.py}).

\paragraph{Deployment consideration.}
Each pass draws a fresh sign pattern, as the recursion requires, but that pattern is shared
across all groups within the pass rather than drawn independently per group. This sharing
preserves the unbiasedness of the per-pass rate but reduces concentration at the layer
level; the model-level average nevertheless tracks $\rho^{P}$ to three decimals
(\cref{fig:overview}B).

\section{Kernel implementation}
\label{app:kernels}

\begin{figure}[t]
\centering
\includegraphics[width=0.85\linewidth]{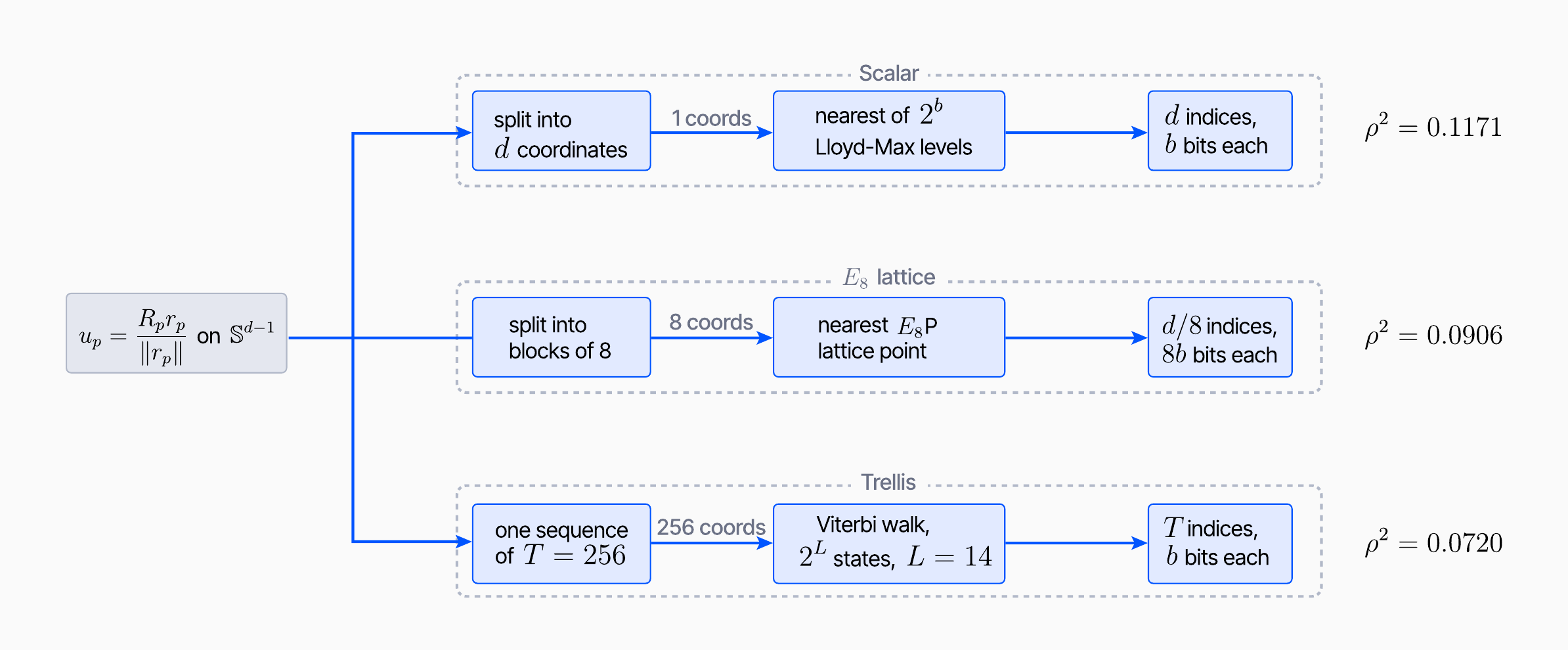}
\caption{The three data-free stages as one dataflow. Every lane takes the same input, the
rotated residual normalized onto the sphere, and every lane emits the same thing, $b$ bits per
weight plus one fp16 scale for the group, so a stage can be swapped for another without
changing the format, the rotation or anything downstream. The lanes differ in exactly one
place, the middle box: how many coordinates are quantized \emph{jointly}. That is the only
axis along which they differ and it is what moves $\rho^{2}$, from $0.1171$ when coordinates
are rounded one at a time to $0.0720$ when $256$ are chosen together by a Viterbi search. The
constants are those of \cref{tab:stages} and are computed from the marginal before any weight
is read; \cref{app:kernels} reports what the trellis lane costs to run.}
\label{fig:stages}
\end{figure}

\paragraph{The Viterbi encoder.}
We contribute a Triton kernel for the Viterbi encoder used by the trellis stage (\cref{fig:stages}). Encoding
is the one step of the pipeline whose cost depends on the weights rather than on the format,
and it is therefore the step that most benefits from acceleration. Trellis coded quantization
represents $T$ coordinates as a walk on a $2^{L}$-state graph, and the minimum-distortion
walk is found by the Viterbi algorithm: at each time step every state selects its best
incoming path by the add-compare-select operation, and the selection is recorded for
traceback at the end. The recursion is serial in time but fully parallel across states, a
structure that maps naturally to a GPU.

The kernel exploits this parallelism by keeping the $2^{L}$ path costs for the
current time step in registers rather than in DRAM. Each Triton program processes a chunk of
weight rows, and within each program the add-compare-select sweep reduces to a register
permutation and a pairwise minimum, avoiding any round trip to global memory. Traceback is
handled by a separate kernel; because each step is a single dependent load, its latency is
hidden by maintaining many rows in flight within one program rather than by launching many
programs. Traceback accounts for approximately four percent of encoding time.

The register-resident design imposes two constraints. First, the path-cost array must fit in
the register file, which limits the trellis width to $2^{L}\le 2^{14}$ in this
implementation. At $L=16$ the array spills to DRAM, the register-resident kernel cannot be
used, and the fallback reference implementation runs at $2.5\times$ rather than $68\times$;
this is why \cref{tab:stages} reports $L=14$ as the shipped width. Second, Triton requires
tensor shapes to be known at compile time, so the state count is fixed as a compile-time
constant and every stage of the sweep operates on identically shaped tensors.

\paragraph{Benchmark protocol.}
To isolate the kernel's contribution we compare it against a reference implementation of the
same Viterbi recursion that stores the cost vector in DRAM.
\Cref{tab:kernel} reports throughput at $d=512$, $b=2$, over one chunk of $1024$ rows on a
single RTX 4090 ($256$ rows at $L=16$, where the register budget is exhausted). Both arms
encode the same weights and are compared on the distortion they achieve rather than on symbol
agreement, since distinct walks of equal path metric are equally valid: add-compare-select
ties are broken differently by the two implementations, producing approximately one differing
symbol per $3000$ at $L=12$ while $\rho^{2}$ agrees to six decimal places.

\begin{table}[h]
\centering
\caption{Viterbi encoder throughput in millions of weights per second, $d=512$, $b=2$,
one RTX 4090. At $L=16$ the path-cost array exceeds the register budget, so the reference
implementation runs and the speedup collapses. Bold marks the shipped configuration rather
than a column best: $L=16$ reaches a lower $\rho^{2}$ but cannot use the register-resident
kernel. $L=14$ is the width \cref{tab:stages} reports and the encoder ships; the tiered
artifacts of \cref{tab:kl} predate the kernel and use $L=12$, which is why that table
carries both. $\rho^{2}$ is quoted to six places here and rounded to four in
\cref{tab:stages}.}
\label{tab:kernel}
\small
\begin{tabular}{@{}rrrrrr@{}}
\toprule
$L$ & states & reference & kernel & speedup & $\rho^{2}$ \\
\midrule
$8$  & $256$    & $16.30$ & $532.06$ & $32.7\times$ & $0.090871$ \\
$10$ & $1{,}024$  & $16.22$ & $304.78$ & $18.8\times$ & $0.080997$ \\
$12$ & $4{,}096$  & $13.38$ & $587.35$ & $43.9\times$ & $0.075239$ \\
$14$ & $16{,}384$ & $2.90$  & $197.61$ & $\mathbf{68.1\times}$ & $\mathbf{0.071999}$ \\
$16$ & $65{,}536$ & $0.72$  & $1.83$   & $2.5\times$  & $0.070240$ \\
\bottomrule
\end{tabular}
\end{table}

\section{Language-model baselines: extended results}
\label{app:llmtable}

\Cref{tab:llmfull-appendix,tab:llmfull-7b} extend \cref{tab:llm-full} with perplexity
and the $99.9$th KL percentile. KL discriminates more sharply than perplexity: at four
bits the arms span $8\%$ in perplexity but a factor of $3.5$ in KL, and on our three
stages the gap is $1$--$3\%$ against $83$--$90\%$. The tail percentile agrees with the
mean KL on the leading two arms in seven of eight columns.

\begin{table}[h]
\centering
\caption{Extended results for the two models near $1$B, adding perplexity and the
$99.9$th-percentile KL to \cref{tab:llm-full}'s mean KL. Bits per weight is measured per
arm; SINQ's bit-width is an upper bound because its dual-scale overhead depends on the
layer shape.}
\label{tab:llmfull-appendix}
\small
\setlength{\tabcolsep}{4pt}
\begin{tabular}{@{}lr|rrr|rrr@{}}
\toprule
& & \multicolumn{3}{c|}{TinyLlama-1.1B} & \multicolumn{3}{c}{Phi-1.5} \\
Method & bits/w & PPL & KL & KL$_{99.9}$ & PPL & KL & KL$_{99.9}$ \\
\midrule
\multicolumn{8}{l}{\emph{$4$-bit class}} \\
QuaRot-RTN & $4.03$ & $8.75$ & $0.116$ & $2.51$ & $22.75$ & $0.059$ & $0.66$ \\
SINQ, $g=64$ & $\leq\!4.53$ & $8.09$ & $\mathbf{0.035}$ & $0.79$ & $22.30$ & $0.028$ & $0.37$ \\
SINQ, $g=512$ & $4.09$ & $8.41$ & $0.069$ & $1.60$ & $22.46$ & $0.050$ & $0.60$ \\
HQQ, $g=64$ & $4.50$ & $8.18$ & $0.046$ & $1.47$ & $22.35$ & $0.033$ & $0.41$ \\
HQQ, $g=512$ & $4.06$ & $8.63$ & $0.095$ & $2.28$ & $22.88$ & $0.068$ & $0.87$ \\
\AQuant{}, scalar & $4.03$ & $8.38$ & $0.074$ & $1.86$ & $22.45$ & $0.040$ & $0.47$ \\
\AQuant{}, $E_8$P & $4.06$ & $8.32$ & $0.062$ & $2.02$ & $22.33$ & $0.035$ & $0.41$ \\
\AQuant{}, trellis & $4.06$ & $8.15$ & $0.040$ & $1.30$ & $22.23$ & $\mathbf{0.022}$ & $0.25$ \\
\AQuant{}, scalar, $P{=}2$ & $4.06$ & $8.82$ & $0.120$ & $3.06$ & $22.64$ & $0.062$ & $0.70$ \\
\AQuant{}, $l_p$ codebook & $4.03$ & $10.10$ & $0.261$ & $4.71$ & $23.51$ & $0.099$ & $1.17$ \\
\midrule
\multicolumn{8}{l}{\emph{$2$-bit class}} \\
QuaRot-RTN & $2.03$ & $2.8{\cdot}10^{4}$ & $8.144$ & $19.13$ & $233.79$ & $2.552$ & $11.14$ \\
SINQ, $g=64$ & $\leq\!2.53$ & $746.14$ & $4.537$ & $14.85$ & $78.34$ & $1.431$ & $8.46$ \\
SINQ, $g=512$ & $2.09$ & $2.2{\cdot}10^{5}$ & $10.230$ & $17.95$ & $1.5{\cdot}10^{3}$ & $4.476$ & $14.91$ \\
HQQ, $g=64$ & $2.50$ & $2.7{\cdot}10^{4}$ & $8.115$ & $17.72$ & $120.49$ & $1.912$ & $9.66$ \\
HQQ, $g=512$ & $2.06$ & $3.0{\cdot}10^{4}$ & $8.223$ & $17.30$ & $1.6{\cdot}10^{4}$ & $6.731$ & $19.26$ \\
\AQuant{}, scalar & $2.03$ & $1.1{\cdot}10^{4}$ & $7.188$ & $18.44$ & $61.52$ & $1.160$ & $8.05$ \\
\AQuant{}, $E_8$P & $2.03$ & $2.6{\cdot}10^{3}$ & $5.756$ & $17.95$ & $38.61$ & $0.693$ & $5.52$ \\
\AQuant{}, trellis & $2.03$ & $196.02$ & $\mathbf{3.201}$ & $13.07$ & $31.55$ & $\mathbf{0.447}$ & $4.20$ \\
\AQuant{}, $l_p$ codebook & $2.03$ & $1.9{\cdot}10^{4}$ & $7.729$ & $18.97$ & $592.86$ & $3.507$ & $14.34$ \\
\bottomrule
\end{tabular}
\end{table}

\begin{table}[h]
\centering
\caption{The same results at $7$B, scored by the protocol described below.}
\label{tab:llmfull-7b}
\small
\setlength{\tabcolsep}{4pt}
\begin{tabular}{@{}lr|rrr|rrr@{}}
\toprule
& & \multicolumn{3}{c|}{Mistral-7B-v0.1} & \multicolumn{3}{c}{Qwen2.5-7B} \\
Method & bits/w & PPL & KL & KL$_{99.9}$ & PPL & KL & KL$_{99.9}$ \\
\midrule
\multicolumn{8}{l}{\emph{$4$-bit class}} \\
QuaRot-RTN & $4.03$ & $5.65$ & $0.054$ & $1.99$ & $7.36$ & $0.072$ & $1.43$ \\
SINQ, $g=64$ & $4.51$ & $5.43$ & $\mathbf{0.019}$ & $0.53$ & $7.16$ & $0.039$ & $0.89$ \\
SINQ, $g=512$ & $4.07$ & $5.52$ & $0.040$ & $1.23$ & $7.51$ & $0.096$ & $2.26$ \\
HQQ, $g=64$ & $4.50$ & $5.46$ & $0.025$ & $0.79$ & $7.19$ & $0.044$ & $1.04$ \\
HQQ, $g=512$ & $4.06$ & $5.57$ & $0.050$ & $1.65$ & $7.48$ & $0.096$ & $1.69$ \\
\AQuant{}, scalar & $4.03$ & $5.51$ & $0.037$ & $1.06$ & $7.21$ & $0.049$ & $1.27$ \\
\AQuant{}, $E_8$P & $4.06$ & $5.47$ & $0.030$ & $0.84$ & $7.17$ & $0.043$ & $0.97$ \\
\AQuant{}, trellis & $4.06$ & $5.43$ & $\mathbf{0.019}$ & $0.57$ & $7.06$ & $\mathbf{0.027}$ & $0.57$ \\
\AQuant{}, scalar, $P{=}2$ & $4.06$ & $5.64$ & $0.057$ & $1.95$ & $7.33$ & $0.072$ & $1.42$ \\
\AQuant{}, $l_p$ codebook & $4.03$ & $6.14$ & $0.143$ & $3.56$ & $7.64$ & $0.112$ & $2.64$ \\
\midrule
\multicolumn{8}{l}{\emph{$2$-bit class}} \\
QuaRot-RTN & $2.03$ & $2.9{\cdot}10^{4}$ & $8.543$ & $16.45$ & $449.53$ & $4.217$ & $14.08$ \\
SINQ, $g=64$ & $2.51$ & $45.20$ & $2.107$ & $10.88$ & $3.1{\cdot}10^{3}$ & $6.152$ & $15.68$ \\
SINQ, $g=512$ & $2.07$ & $3.9{\cdot}10^{4}$ & $8.887$ & $15.60$ & $1.4{\cdot}10^{5}$ & $9.985$ & $19.99$ \\
HQQ, $g=64$ & $2.50$ & $3.2{\cdot}10^{3}$ & $6.361$ & $15.95$ & $402.06$ & $4.091$ & $14.92$ \\
HQQ, $g=512$ & $2.06$ & $2.0{\cdot}10^{4}$ & $8.215$ & $13.34$ & $1.4{\cdot}10^{6}$ & $12.294$ & $21.99$ \\
\AQuant{}, scalar & $2.03$ & $450.06$ & $4.392$ & $15.10$ & $24.81$ & $1.334$ & $9.71$ \\
\AQuant{}, $E_8$P & $2.03$ & $92.93$ & $2.836$ & $11.68$ & $13.01$ & $0.681$ & $8.58$ \\
\AQuant{}, trellis & $2.03$ & $36.60$ & $\mathbf{1.903}$ & $11.15$ & $10.64$ & $\mathbf{0.477}$ & $7.01$ \\
\AQuant{}, $l_p$ codebook & $2.03$ & $1.5{\cdot}10^{4}$ & $7.933$ & $20.57$ & $314.35$ & $3.852$ & $13.31$ \\
\bottomrule
\end{tabular}
\end{table}

\paragraph{Baselines.} RTN-absmax at $g{=}512$ collapses on both $1$B models (perplexity
above $600$ at four bits) and is omitted; QuaRot-RTN at the same group size serves as the
data-free baseline. The QuaRot-RTN row implements only the rotation, not the full
method's Hessian-guided rounding; the SINQ and HQQ rows use those authors'
implementations unmodified. HQQ stores a scale and zero point per group in fp16, raising
its $g{=}64$ arm to $4.50$ bits; SINQ stores a scale per group plus a row and column
vector, putting its $g{=}64$ arm at ${\leq}\,4.53$. \AQuant{} stores one fp16 scale per
group of $512$ per pass, yielding $4.03$ bits at one pass and $4.06$ at two; the scale is
rounded to fp16 before the residual is subtracted, so the rounding error enters the
residual and is quantized by the next pass. The $l_p$ codebook row replaces the
Lloyd--Max objective with HQQ's $p{=}0.7$ penalty; it is worse on all eight columns,
indicating that the group size rather than the objective accounts for HQQ's advantage at
$g{=}64$.

\paragraph{The per-pass rate loss, measured.}
A Gaussian source is successively refinable at zero cost under an optimal stage
\citep{equitz1991successive}. As shown in \cref{fig:rateloss}, our stages do not achieve the rate--distortion bound
$D(R)$; \cref{tab:stages} reports a per-pass rate loss $\Delta_{1}$ of $0.45$ bits for
the scalar, $0.27$ for $E_8$, and $0.10$ for the trellis, and \cref{prop:greedy} shows
that $P$ passes accumulate $P\Delta_{1}$. Two $2$-bit passes are therefore less efficient
than one $4$-bit pass at matched storage, and the gap shrinks as the stage improves. The
table confirms this: on TinyLlama two $2$-bit scalar passes reach KL $0.120$ where one
$4$-bit scalar pass reaches $0.074$, at $4.06$ against $4.03$ stored bits, and the KL
ratio is stable across models at $1.61$, $1.55$, $1.55$ and $1.46$ from TinyLlama to
Qwen. The $P{=}2$ artifact is kept not as a rate--distortion improvement but because it
contains its $2$-bit tier as a free prefix.

\paragraph{Scoring at $7$B.} The $7$B models exceed single-GPU memory for two resident
copies, so the reference's per-window log-probabilities are cached to disk in fp16 and
the quantized arm scored against the cache. On TinyLlama, where both models fit
simultaneously, the two paths agree to the last printed digit.

\paragraph{Classification accuracy.} \Cref{tab:zeroshot} supplements the perplexity and KL
metrics with classification accuracy on three standard benchmarks: ARC-Easy~\citep{clark2018arc}, WinoGrande~\citep{sakaguchi2021winogrande}
and HellaSwag~\citep{zellers2019hellaswag}, evaluating whether the KL ordering from \cref{tab:llm-full} carries over to
downstream tasks.

\begin{table}[h]
\centering
\caption{Classification accuracy (mean of ARC-Easy, WinoGrande, HellaSwag), scored by
length-normalized log-likelihood; chance is $33.3$. Binomial standard error per cell is
$0.58$--$0.69$ points. At four bits all gaps are within noise of full precision. Bold at two bits marks the best data-free arm. The last row reports
the trellis arm's above-chance margin as a percentage of full precision's.}
\label{tab:zeroshot}
\small
\setlength{\tabcolsep}{5pt}
\begin{tabular}{@{}lrrrr@{}}
\toprule
Method & TinyLlama & Phi-1.5 & Mistral-7B & Qwen2.5-7B \\
\midrule
\multicolumn{5}{l}{\emph{Reference}} \\
Full precision           & $55.7$ & $68.5$ & $76.7$ & $75.2$ \\
\midrule
\multicolumn{5}{l}{\emph{$4$-bit class}} \\
HQQ, $g=64$              & $54.9$ & $68.1$ & $75.9$ & $75.0$ \\
\AQuant{}, scalar        & $54.9$ & $68.4$ & $75.1$ & $74.1$ \\
\AQuant{}, trellis       & $55.2$ & $68.1$ & $76.0$ & $74.4$ \\
\midrule
\multicolumn{5}{l}{\emph{$2$-bit class}} \\
HQQ, $g=64$              & $33.7$ & $48.2$ & $33.6$ & $39.7$ \\
\AQuant{}, trellis       & $\mathbf{36.8}$ & $\mathbf{61.8}$ & $\mathbf{50.9}$ & $\mathbf{68.9}$ \\
\quad retained over chance & $15.6$ & $81.0$ & $40.5$ & $84.9$ \\
\bottomrule
\end{tabular}
\end{table}

At four bits all arms fall within $1.6$ points of full precision, confirming that four-bit quantization preserves downstream task quality. At two bits the trellis arm retains between $15.6\%$ and $84.9\%$ of full precision's above-chance margin depending on the model, consistent with the KL results in \cref{tab:llm-full}.

\section{Beyond language: evaluation protocol}
\label{app:nonllm}

This appendix details the evaluation protocol behind \cref{tab:beyond} in
\cref{sec:beyond}. Both Clay v1.5 and UNI are evaluated as frozen encoders with a linear
probe, following the standard protocol in their respective domains, so that the quantization
is the only variable across arms. Embeddings are z-scored and a multinomial logistic
regression is fitted by L-BFGS; every arm sees the same inputs and the same probe recipe.

Clay is evaluated on $10{,}000$ EuroSAT chips over ten classes in an $80/20$ train/test
split, read from the thirteen-band L1C product and reduced to the ten Sentinel-2 wavelengths
the model expects. UNI is evaluated on $8{,}000$ training and $8{,}000$ test PatchCamelyon
patches, resized from $96\times96$ to $224\times224$. Under this protocol UNI's
full-precision probe reaches $87.81\%$, below its published accuracy, because the probe is
fitted on $8{,}000$ patches rather than the full $262{,}144$-patch training split. The
comparison across tiers is unaffected, since all arms share the protocol, but the absolute
figure should not be taken as UNI's published PatchCamelyon result.

For each model \cref{tab:beyond} reports two evaluation modes. In the task column the probe
is refitted at every tier, measuring how well a fresh classifier can decode the quantized
representation. In the KL column the probe is fitted once on full-precision features and the
same head is applied unchanged to each tier, measuring whether the quantized encoder still
produces the representation the head was trained on rather than merely one a new head can
recover.

All $96$ of UNI's transformer linear layers are quantized while its single convolution, the
patch projection, is held dense. Clay likewise quantizes its $96$ transformer linears at two
bits with group size $512$, holding its seven patch-embedding linears dense, six of which
have fan-in $128$ and therefore admit no full-width group. As \cref{sec:beyond} notes, the
KL column on both encoders falls by roughly an order of magnitude per tier, consistent with
the geometric decay of \cref{prop:decay}.

\section{Stage ordering end to end}
\label{app:stageorder}

\Cref{tab:stages} assigns each stage a data-free constant $\rho$ that measures weight-space
distortion per pass. \Cref{tab:kl} tests whether the ordering induced by $\rho$ carries
through to an end-to-end metric by scoring all four stages on Gemma-3-12B \citep{team2025gemma} under identical
conditions, varying only the stage. In weight space the ordering is confirmed: the measured
constants $\rho^{2}=0.1167$, $0.0907$, $0.0771$ and $0.0742$ on Gemma's
$\mathrm{q\_proj}$ reproduce \cref{tab:stages}'s scalar and $E_8$P rows to three decimal
places. The two trellis rows cannot be compared as directly: \cref{tab:stages} reports only $L=14$,
so no tabulated value exists for $L=12$, and the $L=14$ measurement on these weights is
$3.1\%$ above the tabulated $0.0720$. The four constants nevertheless establish a clear ordering in weight space; what remains to
be determined is whether that ordering carries through to the model's output.

Because all tiers are scored on the same $32$ windows of $2048$ tokens, the comparison is
paired. Pairing is essential here, as the variance across windows substantially exceeds the
differences between stages. At $P{=}1$ every pair separates: the $L=14$
trellis produces the lowest KL and the scalar stage the highest, both by wide margins, while
$E_8$P falls below the $L=12$ trellis by $0.348$ nats ($t=2.38$, lower on $21$ of $32$
windows), the one point at which the measured order departs from $\rho$'s prediction. At
$P{=}2$ both trellis rows fall below $E_8$P and the scalar stage, but $E_8$P and the scalar
stage are no longer separated from each other ($t=0.52$), nor are the two trellis rows
($t=1.19$). At $P{=}3$ only the scalar stage remains distinguishable: the three vector
stages are statistically tied, and the $L=14$ row carries the higher mean despite being
lower than $L=12$ on $19$ of $32$ windows, its mean elevated by a heavier tail. We
report the ordering only at tiers where the paired test resolves it. The overall pattern
is monotone: as the bit budget grows the stage's influence on end-to-end KL progressively
diminishes, and at six bits per weight no pair of stages remains statistically
distinguishable.

\begin{table}[t]
\centering
\caption{Gemma-3-12B, $\mathrm{KL}(p_{\mathrm{bf16}}\,\|\,p_{\mathrm{quant}})$ in nats per
token on held-out WikiText-2, lower is better, reported as mean (standard error) over $32$
windows of $2048$ tokens. Every arm is data-free and differs only in the stage of
\cref{tab:stages}; $g=256$. The unquantized reference is scored once with its tower
offloaded and its log-probabilities cached; all four quantized rows are then evaluated
against that single cache, ensuring an identical reference across rows
(\texttt{experiments/kl\_two\_phase.py}). The $P{=}1$ and $P{=}2$ columns are prefixes of
the same $P{=}3$ artifact, verified to be bit-identical to encoding at that depth directly
(\texttt{experiments/prefix\_identity\_gemma.py}). The first three rows read the published
artifacts; the $L=14$ row is quantized here with all other settings matched, because
\cref{tab:stages} reports $L=14$ while the published artifact uses $L=12$. No cell is bolded,
as the stages are separated at $2$ bits but not at $6$
(\texttt{experiments/kl\_separability.py}).}
\label{tab:kl}
\small
\begin{tabular}{@{}lrrr@{}}
\toprule
Stage & $P=1$ ($2.06$ bit) & $P=2$ ($4.125$ bit) & $P=3$ ($6.19$ bit)\\
\midrule
Scalar Lloyd--Max      & $6.472$ \small{$(0.131)$} & $0.4157$ \small{$(0.0324)$} & $0.1154$ \small{$(0.0194)$}\\
$E_8$P lattice         & $3.648$ \small{$(0.120)$} & $0.3850$ \small{$(0.0533)$} & $0.0502$ \small{$(0.0070)$}\\
Trellis, 1MAD ($L=12$) & $3.996$ \small{$(0.168)$} & $0.2626$ \small{$(0.0300)$} & $0.0440$ \small{$(0.0075)$}\\
Trellis, 1MAD ($L=14$) & $3.068$ \small{$(0.134)$} & $0.2312$ \small{$(0.0279)$} & $0.0581$ \small{$(0.0151)$}\\
\bottomrule
\end{tabular}
\end{table}

\section{Comparison with elastic and static quantization baselines}
\label{app:elastic}

This appendix compares \AQuant{} against published elastic and static methods on
LLaMA2-7B \citep{touvron2023llama2} and LLaMA3-8B \citep{grattafiori2024llama3}, the two models the elastic methods all report. \Cref{tab:elastic} collects the results. The MoBiQuant
and AnyBCQ \citep{park2026anybcq} rows are from \citet{wang2026mobiquant}'s tables without reproduction, as
neither has released code; MoBiQuant's bit-widths are averages over routed tokens whose
router parameters are excluded from the reported rate. Because no published row reports KL
divergence, the table uses perplexity throughout; on our rows the two metrics produce the
same ranking. Our measured full-precision perplexity is $5.503$ against $5.5$ printed on
LLaMA2-7B and $6.1982$ against $6.1$ on LLaMA3-8B; the LLaMA3 discrepancy exceeds any
four-bit difference in \cref{tab:elastic}, so that column supports an ordering but not a
precise margin.

\begin{table}[t]
\centering
\caption{WikiText-2 perplexity at sequence length $2048$; lower is better. Our rows use
the harness of \cref{tab:llm-full}; all others are from \citet{wang2026mobiquant}
(Tables~1 and~2). Bit-widths for published elastic rows are averages over routed tokens
excluding router parameters; ours are stored bits per weight ($4.03$ scalar, $4.06$
vector, $2.03$ throughout). Bold marks the lowest quantized entry per column, excluding
full precision. $^{*}$Reproduced using the released quantizer and evaluation script
(\texttt{experiments/anyprecision\_\allowbreak llama2\_\allowbreak 7b.txt}); the $6.05$
and $2\mathrm{e}3$ in \citet{wang2026mobiquant} do not match the $5.624$ and $35.21$ we
obtain.}
\label{tab:elastic}
\small
\setlength{\tabcolsep}{5pt}
\begin{tabular}{@{}llc|rr|rr@{}}
\toprule
& & & \multicolumn{2}{c|}{LLaMA2-7B} & \multicolumn{2}{c}{LLaMA3-8B} \\
Method & Data & Elast. & $\approx\!4$ bit & $\approx\!2$ bit & $\approx\!4$ bit & $\approx\!2$ bit \\
\midrule
Full precision            & ---                  & --- & $5.503$ & $5.503$ & $6.198$ & $6.198$ \\
\midrule
\AQuant{} trellis (ours)  & none                 & yes & $5.613$ & $71.80$ & $6.550$ & $52.51$ \\
\AQuant{} $E_8$P (ours)   & none                 & yes & $5.723$ & $218.7$ & $6.735$ & $165.7$ \\
\AQuant{} scalar (ours)   & none                 & yes & $5.750$ & $2046$  & $6.817$ & $3134$ \\
\midrule
MoBiQuant                 & $128$ seq., $20$ ep. & yes & $5.82$  & $10.91$ & $7.31$  & $58.12$ \\
AnyBCQ                    & calibration set      & yes & $5.91$  & $15.38$ & $6.84$  & $19.01$ \\
Any-Precision LLM$^{*}$   & calibration set      & yes & $5.624$ & $35.21$ & $6.70$ & $1\mathrm{e}3$ \\
\midrule
QTIP                      & Hessian              & no  & $\mathbf{5.17}$ & $\mathbf{5.86}$ & $\mathbf{6.06}$ & $\mathbf{7.25}$ \\
QuIP\#                    & Hessian              & no  & $5.66$  & $12.3$  & $6.64$  & $15.20$ \\
SpinQuant                 & calibration set      & no  & $5.60$  & ---     & $6.40$  & --- \\
QuaRot                    & calibration set      & no  & $5.61$  & ---     & $7.20$  & --- \\
\bottomrule
\end{tabular}
\end{table}

\paragraph{Any-Precision LLM.} Of the three elastic methods only Any-Precision LLM has
released code. Quantizing with its authors' pipeline reproduces their published $5.62$ at
four bits as $5.624$ and their $6.13$ at three bits as $6.153$; the $6.05$ and
$2\mathrm{e}3$ that \citet{wang2026mobiquant} reports do not match, the two-bit
discrepancy being a factor of $57$. At four bits \AQuant{} achieves $5.613$ against Any-Precision LLM's $5.624$; at two bits
Any-Precision LLM achieves $35.21$ against our $71.80$. \AQuant{} is data-free while Any-Precision LLM uses a
calibration set. Against MoBiQuant and AnyBCQ \AQuant{} achieves lower four-bit
perplexity on both models ($5.613$ against $5.82$ and $5.91$ on LLaMA2-7B, $6.550$
against $7.31$ and $6.84$ on LLaMA3-8B), though these margins should be interpreted with
caution given the discrepancies observed under reproduction.

\paragraph{Static methods and QTIP.} At 4bits, on LLaMA2-7B \AQuant{} achieves $5.613$, within $0.01$ of SpinQuant's $5.60$ \citep{liu2024spinquant} and
QuaRot's $5.61$ \citep{ashkboos2024quarot}, and falls below QuIP\#'s $5.66$. On
LLaMA3-8B it falls below QuIP\# ($6.550$ against $6.64$) but lies above SpinQuant's
$6.40$. QTIP achieves lower perplexity on both models: by $0.44$--$0.49$ points at four
bits and by a factor of seven to twelve at two bits. The four-bit gap reflects the absence of interblock feedback: QTIP propagates
Hessian-guided corrections across blocks, whereas \AQuant{} rounds each group
independently. At two bits all methods except
QTIP produce unusable models on both architectures. \Cref{tab:zeroshot} reports the same
pattern on downstream tasks.

\AQuant{} is the only elastic method in \cref{tab:elastic} that requires no data, and the
only one whose lower-rate artifact is a prefix of the higher-rate one, incurring no
additional cost.

\section{Nested decoding in service}
\label{app:serving}

\begin{figure}[t]
\centering
\includegraphics[width=\linewidth]{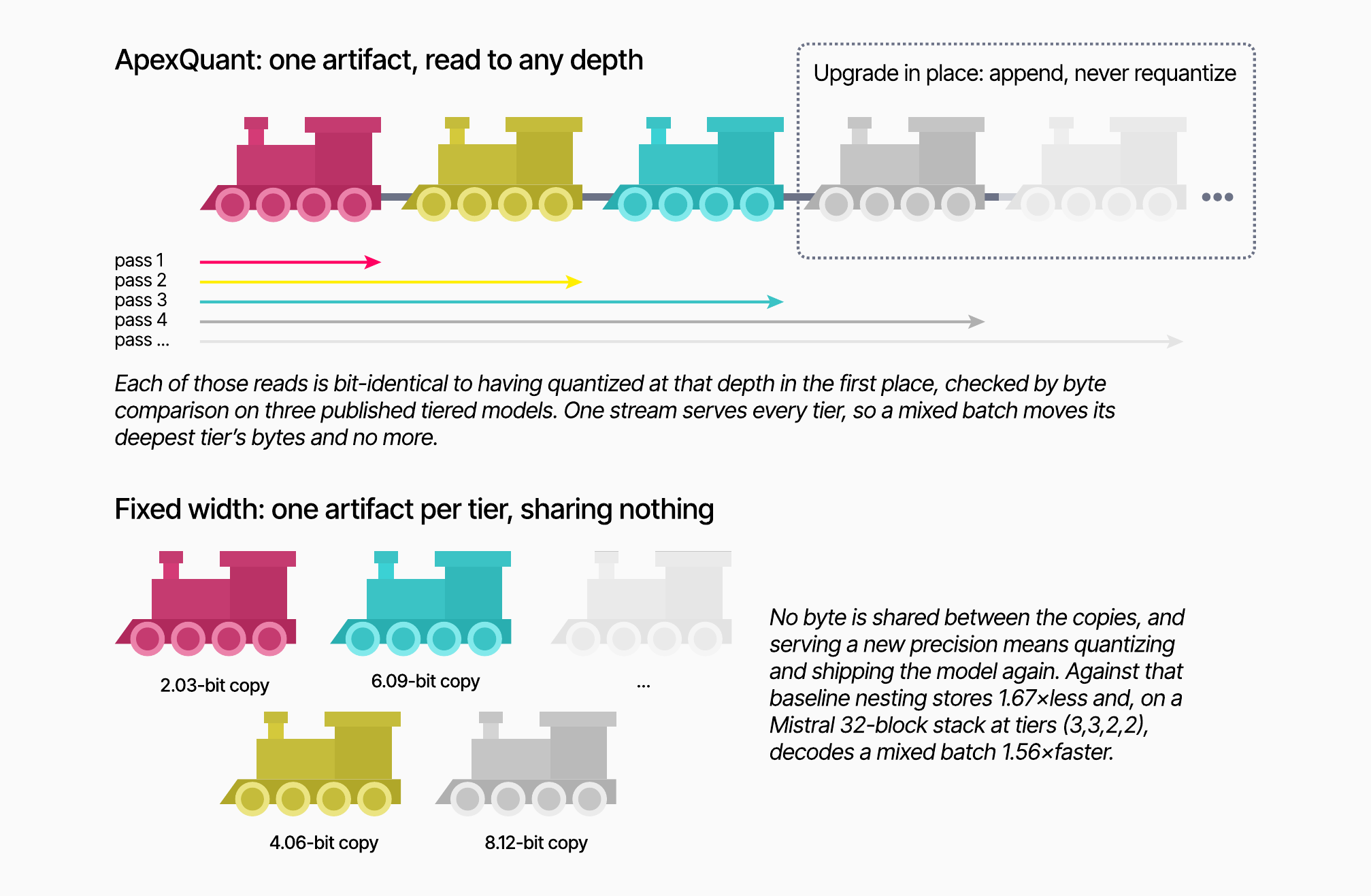}
\caption{What ``elastic'' means in bytes. The passes of one weight row are stored back to back,
so the first $P'$ of them are a contiguous prefix: reading it yields exactly the artifact a
$P'$-pass run would have produced, checked by byte comparison on three published tiered models,
and raising a deployed model's rate appends indices rather than rebuilding it. The reads shown
yield $2.03$, $4.06$ and $6.09$ bits per weight at $P'=1,2,3$, and the one stream that serves
them all costs its deepest tier alone. Fixed-width elasticity has to keep one independent copy
per tier: on the Mistral $32$-block stack at tiers $(3,3,2,2)$ that is a $6.09$ and a $4.06$
copy, $10.16$ bits per weight against nesting's $6.09$, which is where the $1.67\times$ storage
factor comes from. That factor holds unconditionally, while the $1.56\times$ latency factor
needs a batch deep enough for the extra passes to stop being free (\cref{app:serving}).}
\label{fig:nested}
\end{figure}

This appendix details the serving measurements behind \cref{tab:serving} and complements
the elastic storage discussion in \cref{sec:elasticmeasured} and \cref{fig:nested}.
Each pass of the quantizer appends a fixed number of index bits to the stored payload, and
because each pass quantizes the current residual under a rotation drawn independently of all
subsequent passes, the indices written by pass $t$ are fully determined before pass $t{+}1$ is
ever computed. Adding a further pass therefore appends to the byte stream but never alters any
previously written byte, which means that the first $t$ passes of a $P$-pass artifact are
bit-for-bit identical to the output of a standalone $t$-pass quantization.

This nesting property enables mixed-precision batched serving from a single weight stream.
Rows in the batch are sorted by descending tier so that the consumers of pass $p$ form a
contiguous prefix, and one split-K kernel launch per pass covers exactly those rows. Each pass
is read from memory once and applied to every row whose tier reaches it, so the batch's total
weight traffic equals that of its deepest tier alone, rather than the sum over individual rows.

Correctness of the batched path rests on two empirically verified properties. The first is
data nesting as described above: truncating the stored stream to $t$ passes must yield the
same bytes as a standalone $t$-pass quantization, so that every row in the batch receives
exactly the indices it would have received had it been quantized independently at its requested
pass count. The second is kernel correctness: the fused tiered kernel processes the entire
mixed batch in a single launch, and for each row it must produce the same dequantized output
as loading that row's weight at its requested pass count and multiplying against the activation
in isolation. We compare these two outputs row by row and observe worst-case agreement to
$1.8\times10^{-3}$, a tolerance consistent with floating-point non-associativity arising from
the different accumulation orders in the batched and per-row code paths.

We verify data nesting against the published artifacts on the Hub. For each of the three model
families released as tiered artifacts, Llama-3.1-70B-Instruct, Qwen3.5-35B-A3B and
Gemma-3-12B, separate repositories are published at each pass count: three passes storing
$6.09$ bits per weight, two passes at $4.06$, and one pass at $2.03$. The test reads one shard
of the three-pass repository, truncates to its first two passes, and compares the result byte
for byte against the standalone two-pass repository; it then repeats the comparison at one
pass against the one-pass repository. All $213$ code tensors and $213$ scale tensors agree
with no tolerance, and each tier's recorded bits per weight matches $P'(b + 16/g)$ exactly.
Any tensor that the format passes through unquantized is required to be identical as well, so
a slicing bug cannot hide in the metadata. The script is
\texttt{serving\_evals/prefix\_check.py}, and it reports a failure list rather than a summary,
because a check that compares nothing must not print success.

\Cref{tab:serving} measures the latency consequence on a Mistral $32$-block stack at tiers
$(3,3,2,2)$, graph-captured so the comparison is not distorted by launch overhead. Eager
single-round timing reverses the sign of this comparison, so the graph capture is necessary.

The scope of this measurement is narrower than a full decode latency. The harness times
the five linear shapes of one Mistral block at a single decode step, on random weights, and
multiplies by the $32$-block depth; there is no attention, no KV cache, no normalization,
rotary, embedding or sampling in the reported number, and no sequence length, since none
enters. Every row uses our own kernels rather than an external baseline, so the comparisons
are internal: they establish that nesting achieves lower latency than storing one artifact per
tier \emph{through the same code path}, which is the claim being made, and not that either
configuration is faster than a production serving stack. The $85\%$ of roofline quoted below
is likewise a weight-traffic roofline of the quantized format rather than of the dense model.

\begin{table}[h]
\centering
\caption{Decode latency for one batch, Mistral $32$-block stack, tiers $(3,3,2,2)$,
graph-captured under two autotune scopes. $M$ is the batch size. Row D is ours.}
\label{tab:serving}
\small
\begin{tabular}{@{}llrr@{}}
\toprule
& Scenario & fast (ms) & max (ms) \\
\midrule
A & Naive sequential, four batch-1 decodes           & $28.11$ & $28.19$ \\
B & Premium-only batch, $M{=}2$, three passes        & $19.38$ & $13.46$ \\
C & Homogeneous premium batch, $M{=}4$               & $19.20$ & $14.02$ \\
D & \textbf{Tiered batch, $M{=}4$, mixed (ours)}     & $\mathbf{19.47}$ & $\mathbf{13.95}$ \\
E & Separate artifacts, three-pass plus a distinct two-pass & $31.81$ & $21.76$ \\
\bottomrule
\end{tabular}
\end{table}

We report three readings from \cref{tab:serving}, all at the max autotune scope.

Adding two standard-tier requests to a premium batch incurs a marginal cost of
$D-B = 0.49$\,ms, against roughly $9$\,ms for those two requests to decode on their own,
approximately one eighteenth of their standalone cost. A sweep over every mixture at $M=4$ and
$M=8$ places the riders' marginal cost between $0.7\%$ and $6.9\%$ of standalone, with one
cell at $-5.0\%$ that we read as noise about zero rather than as a saving, so the ratio is
representative rather than a chosen operating point.

Serving a mixed-quality batch incurs the same cost as a homogeneous premium batch:
$D/C = 0.995$, and the mixture is if anything marginally cheaper because the deepest pass
covers fewer rows.

Nesting also achieves lower latency and smaller storage than fixed-width elasticity. Against
storing one artifact per tier, which shares no bytes between tiers, nesting is $1.56\times$
faster ($E/D$) and $1.67\times$ smaller ($6.09$ against $10.16$ bits per weight), with
kernel-identical comparisons. Serving the tiers sequentially instead incurs a $2.02\times$
penalty ($A/D$).

At matched storage, a single wider pass achieves lower distortion than multiple narrow passes,
because each stage quantizes above the rate--distortion bound and the per-pass rate loss
$\Delta_1$ reported in \cref{tab:stages} therefore accumulates linearly with the number of
passes (\cref{prop:greedy}); \cref{app:llmtable} confirms the ordering on every model. The
nested artifact is therefore not retained as a rate--distortion improvement over a single wider
quantizer at the same bit budget. Its value is operational: one stored artifact serves every
tier from premium to standard as a contiguous prefix read, upgrading precision is an append
rather than a requantization, and the $1.67\times$ storage saving over fixed-width elasticity
holds regardless of the batch regime. These properties have no analogue in a single-rate
quantizer, which must be rebuilt entirely to change its operating point.

The scope sensitivity of these results bounds the claim in two ways. Triton's wider autotune
sweep (the max column of \cref{tab:serving}) improves every mid-$M$ scenario by roughly
$1.4\times$ over the narrower sweep, but yields no further improvement on the batch-1
sequential path, whose three-pass token already utilizes $85\%$ of the available memory
bandwidth for weight reads; its one- and two-pass tokens reach $63\%$ and $77\%$. Inside the
memory-bound window, serving every request at premium precision is as cheap as tiering, so the
latency advantage requires a batch large enough that the premium tier's extra passes are no
longer free. What tiering provides unconditionally is the $1.67\times$ storage factor and the
ability to serve a mixture from one artifact.

Every scenario in \cref{tab:serving} is on the decode path, where the arithmetic is
memory-bound and the extra passes ride in the shadow of the weight read. Prefill is
compute-bound, so there the passes lie on the critical path and a $P$-pass layer performs $P$
times the dequantize-multiply work; we measure no prefill scenario and claim none.

\paragraph{Distinction from speculative decoding.}
A speculative draft is a separate, smaller model with its own weights, and correctness rests
on an acceptance test that preserves the target distribution, with a rollback when the draft
is rejected. A nested prefix is the \emph{same} weights read to fewer passes, so it requires
no additional storage, and there is no verifier and nothing to roll back: a request is served
at the precision it asked for. A $P'$-prefix can also act as the draft for a $P$-pass target
without requiring extra storage. We evaluate that composition: with our own kernels making the
target $1.9\times$ faster, the measured cycle runs at $0.67$ to $0.76\times$ of plain
decoding, so prefix speculation does not pay. The reason is structural and persists regardless
of engineering improvements: progress on the target raises the bar the draft must clear in
exact proportion, so the faster the kernels the worse weight-prefix speculation performs. The
elasticity and the kernels are the contribution; the draft reuse is not.

\section{The calibration pass}
\label{app:calib}

\begin{figure}[t]
\centering
\includegraphics[width=\linewidth]{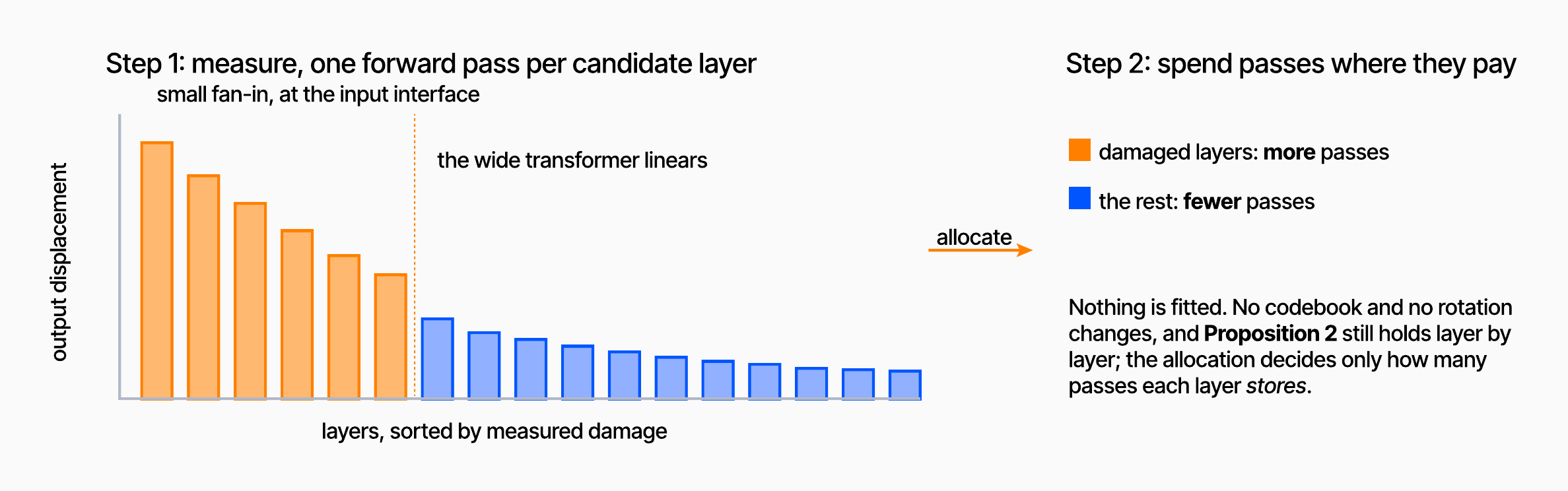}
\caption{The calibration pass, the single exception to reading only weights. Layers differ in
how much a given quantization error hurts the model, so the allocation measures that directly
rather than inferring it: each candidate layer is quantized alone with the rest left at full
precision and the displacement of the model's output is recorded, $97$ forward passes on Clay
rather than one in total. The ranking is then spent as pass counts, and nothing is refitted.
What it buys is a strict Pareto improvement in \emph{embedding} error, $0.0406$ at $4.99$ bits
per weight against $0.0534$ at $6.19$ for uniform: lower error \emph{and} less storage. What it
does not buy is the downstream task, where the allocated arm sits $0.15$ points \emph{below}
uniform on the EuroSAT probe, about a third of one standard error, for $0.8$ more stored bits;
on language models the sign flips between models. The gain concentrates on a small slice of
layers rather than spreading evenly, which is the structure an unweighted weight-space
objective cannot see. Which slice we do not claim: this run keeps the patch-embedding linears
dense throughout, so it is not the input interface the panel labels (\cref{app:calib}).}
\label{fig:allocation}
\end{figure}

\paragraph{Mechanism.}
The calibration pass is the single optional step that reads data (\cref{fig:allocation}).
Sensitivity is measured by a leave-one-out protocol: for each candidate layer $l$ the
encoder runs with $l$ quantized and every other layer at full precision, and the displacement
of the model's output is recorded. This requires one forward pass per candidate layer plus
one full-precision reference, totalling $97$ forwards on Clay's $96$-layer encoder.

The resulting ranking is consumed by a greedy allocator: all layers receive a base number of
passes, and the top-$k\%$ most sensitive layers receive additional passes. For example,
``base $2$ $+3$ at $25\%$'' assigns two passes to every layer and three extra passes to the
most sensitive quarter. No codebook is re-estimated, no rotation changes, and
\cref{prop:decay} continues to hold per layer; the allocation decides only how many passes
each layer stores.

\paragraph{Results on the vision encoder.}
\Cref{fig:allocation} summarises the outcome on Clay v1.5, measured end to end as embedding
relative error against the full-precision forward:

\begin{center}
\small
\begin{tabular}{@{}lrr@{}}
\toprule
Allocation & avg bits/weight & rel.\ embedding error \\
\midrule
uniform, one pass                & $2.06$ & $0.4990$ \\
uniform, two passes              & $4.13$ & $0.1576$ \\
uniform, three passes            & $6.19$ & $0.0534$ \\
\midrule
calibrated, base $2$ $+2$ at $10\%$ & $4.48$ & $0.0831$ \\
calibrated, base $2$ $+2$ at $15\%$ & $4.60$ & $0.0652$ \\
calibrated, base $2$ $+3$ at $25\%$ & $\mathbf{4.99}$ & $\mathbf{0.0406}$ \\
\bottomrule
\end{tabular}
\end{center}

The last row Pareto-dominates the uniform three-pass allocation: lower error ($0.0406$
against $0.0534$) at $1.2$ fewer bits per weight. The ranking is stable across probe draws,
with Spearman correlation between $0.978$ and $0.987$ over eight independent synthetic
inputs and the top-$76$ layer set identical across all pairs. However, on the downstream
EuroSAT probe of \cref{sec:beyond} the allocation does not improve accuracy: the allocated
arm reaches $96.10\%$ against the uniform two-pass arm's $96.25\%$ at $4.92$ against $4.13$
stored bits, a difference well within the test set's standard error of $0.43$ points. The
Pareto gain is therefore confined to embedding error and does not extend to downstream task
accuracy on this benchmark.

\paragraph{Failure on language models.}
The same approach does not transfer to language models. On WikiText-2 KL against FP16,
measured over eight chunks at sequence length $2048$ with the $2$-bit scalar stage at
$g=512$:

\begin{center}
\small
\begin{tabular}{@{}llrrr@{}}
\toprule
Model & Budget & uniform KL & allocated KL & ratio \\
\midrule
\multirow{2}{*}{TinyLlama-1.1B}
 & $4.06$ bit & $0.1276$ & $\mathbf{0.0817}$ & $1.56\times$ better \\
 & $6.09$ bit & $0.0120$ & $\mathbf{0.0067}$ & $1.78\times$ better \\
\midrule
\multirow{2}{*}{Phi-1.5}
 & $4.06$ bit & $\mathbf{0.0573}$ & $0.0951$ & $1.66\times$ worse \\
 & $6.09$ bit & $\mathbf{0.0064}$ & $0.0077$ & $1.20\times$ worse \\
\bottomrule
\end{tabular}
\end{center}

On TinyLlama the allocation improves KL at both budgets, but on Phi-1.5 at four bits the
allocated artifact is $1.66\times$ worse than uniform. A validation gate that compares the
two and retains the allocation only when it improves correctly rejects this case and falls
back to uniform. At six bits the degradation is smaller ($1.20\times$), within the range a
single-evaluation gate does not reliably detect, and the allocation is shipped. A
multi-evaluation gate or a more robust sensitivity estimator could close this gap, but we
leave both to future work and restrict the calibration pass to the vision encoder in its
current form.

\end{document}